%% file: paper.tex
\documentclass[twocolumn, apalike]{article}

\usepackage[pdftex]{graphicx}
\usepackage{amsfonts}
\usepackage{amsmath, latexsym, amssymb, bm, ascmac, here, mathtools,comment}
\usepackage{amsthm}
\usepackage{algorithm, algorithmic}
\usepackage{mathrsfs}

\newtheorem{theorem}{Theorem}[section]

\newtheorem{lemma}{Lemma}[section]

\renewcommand{\refname}{References}

\newcommand{\N}{\mathbb{N}}

\newcommand{\R}{\mathbb{R}}

\newcommand{\D}{\mathcal{D}}
\newcommand{\E}{{\rm{E}}}
\newcommand{\V}{{\rm{Var}}}
\newcommand{\Cov}{{\rm{Cov}}}
\newcommand{\PR}{{\rm{P}}}
\newcommand{\q}{\quad}

\newcommand{\f}{^\forall}
\newcommand{\bx}{{\bm{x}}}

\newcommand{\argmax}{\operatornamewithlimits{argmax}}

\newcommand{\1}{\mbox{1}\hspace{-0.25em}\mbox{l}}

\title{Bayesian Experimental Design for Finding Reliable Level Set under Input Uncertainty}
\date{}
\author{Shogo Iwazaki \thanks{Nagoya Institute of Technology}
    \and Yu Inatsu \thanks{RIKEN Center for Advanced Intelligence Project}
    \and Ichiro Takeuchi \footnotemark[1] \footnotemark[2] \thanks{National Institute for Materials Sciences} \thanks{email:takeuchi.ichiro@nitech.ac.jp}}
\begin{document}
%


\maketitle
\input{abstruct}
\input{section1}
\input{section2}
\input{section3}

\input{section4}
\input{section5}
\input{section6}
\input{acknowledements}
\input{bib}
\clearpage
\clearpage
\input{appendix}
\input{sec2-utf8}
\input{sec3-utf8}
\input{subsec3-1-utf8}
\input{subsec3-3-utf8}
\end{document}

%% file: abstruct.tex
\begin{abstract}
In the manufacturing industry, it is often necessary to repeat expensive operational testing of  machine in order to identify the range of input conditions under which the machine operates properly.
Since it is often difficult to accurately control the input conditions during the actual usage of the machine, there is a need to guarantee the performance of the machine after properly incorporating the possible variation in input conditions.
In this paper, we formulate this practical manufacturing scenario as an \emph{Input Uncertain Reliable Level Set Estimation (IU-rLSE)} problem, and provide an efficient algorithm for solving it.
The goal of IU-rLSE is to identify the input range in which the outputs smaller/greater than a desired threshold can be obtained with high probability when the input uncertainty is properly taken into consideration.
We propose an active learning method to solve the IU-rLSE problem efficiently, theoretically analyze its accuracy and convergence, and illustrate its empirical performance through numerical experiments on artificial and real data.
\end{abstract}

%% file: section1.tex
\section{Introduction}\label{sec:intro}
In the manufacturing industry, it is often necessary to repeat operational testing of machine in order to identify the range of input conditions under which the machine operates properly.
When the cost of an operational test is expensive, it is desirable to be able to identify the region of appropriate input conditions in as few operational tests as possible.
If we regard the operational conditions as inputs and the results of the operational tests as outputs of a black-box function, this problem can be viewed as a type of active learning (AL) problem called \emph{Level Set Estimation (LSE)}.
LSE is defined as the problem of identifying the input region in which the outputs of a function are smaller/greater than a certain threshold.
In the statistics and machine learning literature, many methods for the LSE problem have been proposed~\cite{bryan2006active,gotovos2013active,zanette2018robust}.

In practical manufacturing applications, since it is often difficult to accurately control the input conditions during the actual usage of the machine, there is a need to guarantee the performance of the machine after properly incorporating the possible variation of input conditions.
In this paper, we formulate this practical manufacturing problem as an \emph{Input Uncertain Reliable Level Set Estimation (IU-rLSE)} problem, and provide an efficient algorithm for solving it.
The goal of IU-rLSE is to identify the input region in which the probability of observing an output smaller than a specified threshold is sufficiently large, when the input uncertainty is taken into account.
Figure \ref{fig:setting_image} illustrate the basic idea of IU-rLSE problem.

We define the reliability of an input point as the probability of observing outputs smaller than a specified threshold, and the reliable input region as the subset of the input region in which the reliability is greater than a certain probability threshold (e.g., 0.95).
Under the assumption that the prior distribution of the true function follows a Gaussian Process (GP), we propose a novel Bayesian experimental design (c.f., active learning) method to identify the reliable input region in as few function evaluations as possible, and call the method the \emph{IU-rLSE} method (with slight abuse of terminology).
Specifically, we extend an acquisition function (AF) from an ordinary LSE problem so that the input uncertainty is properly taken into account, and develop a reasonable approximation of the AF for which expensive integral calculations are necessary unless our approximation is used.
We theoretically analyze the accuracy and convergence of the proposed IU-rLSE method, and illustrate its numerical performance by applying the method to both synthetic and real datasets.

\begin{figure*}[t]
    \begin{center}
        \includegraphics[scale=0.5]{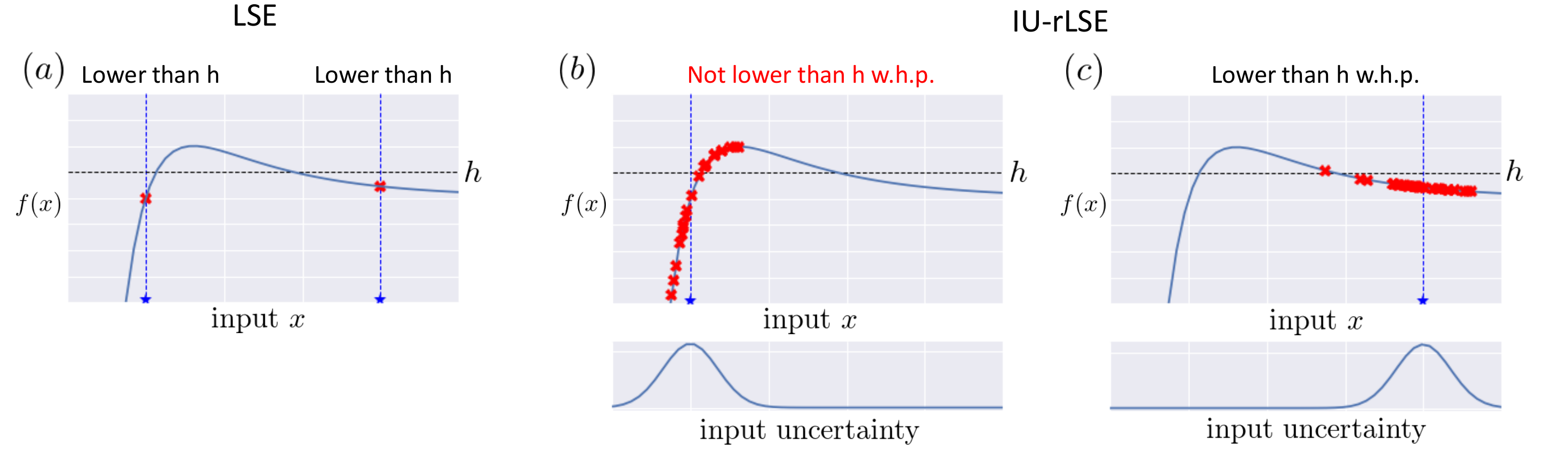}
        \caption{
     An illustrative example of IU-rLSE problem.
     ({\bf a}) An example of ordinary LSE problem. The two input points (blue stars) are considered as appropriate input points because the corresponding outputs are smaller than the desired threshold $h$.
     ({\bf b}) and ({\bf c}) Examples of IU-rLSE problem.
     In IU-rLSE problems, when a user specifies input points as indicated by bule stars, due to the input uncertainty, actual inputs are variated and hence the observed outputs are also variated as indicated by red crosses.
     In ({\bf b}), the probability of observing outputs smaller than the threshold $h$ (66\%) is not sufficiently high, and thus the input point (blue star) is not considered as an appropriate input point when the variability is taken into consideration.
     On the other hand, in ({\bf c}), the probability of observing outputs smaller than the threshold $h$ (97\%) is sufficiently high, and thus the input point (blue star) is considered as an appropriate input point even when the variability is taken into consideration.
     }
     \label{fig:setting_image}
    \end{center}
\end{figure*}

\paragraph{Related Work}
Machine learning problems for black-box functions with high evaluation cost have been studied in the context of \emph{active learning (AL)}~\cite{settles2009active}.
The problem of finding the global optimal solution for black-box functions is called Bayesian Optimization (BO)~\cite{shahriari2016taking}.
In BO and related AL problems, Gaussian Process (GP) model is often used as a nonparametric and flexible model of black box functions.
GP model was first used for LSE problem in \cite{bryan2006active}, where the authors proposed an AF based on \emph{Straddle heuristic}.
Then, \cite{gotovos2013active} proposed a new AF based on GP-UCB~\cite{srinivas2009gaussian} framework, and prove the convergence of the algorithm.
Recently, \cite{zanette2018robust} proposed another new AF for LSE problem based on expected improvement of classification accuracy.
LSE problems are also used in the context of safe BO~\cite{sui2015safe,DBLP:conf/icml/SuiZBY18}.
Furthermore, \cite{bogunovic2016truncated} introduced a unified framework of BO and LSE problems.
In order to obtain the predictive distribution of GP model under input uncertainty, integral calculations of the GP model over the input distribution is necessary.
%
Integral calculation on GP models have been studied in various contexts~\cite{girard2003gaussian,o1991bayes,xi2018bayesian,gessner2019active}.
In the context of AL such as BO, there are some studies dealing with input uncertainty~\cite{beland2017bayesian,oliveira2019bayesian,inatsu2019active}, but none of them consider the same problem setup as ours.
\paragraph{Contribution}
Our main contributions in this paper are as follows:
\begin{itemize}
 \item
      Assuming GP model as a prior distribution of the true function
      $f$,
      we formulate IU-rLSE problem,
      i.e.,
      the problem of identifying the set of input points
      at which
      the probability of observing a response smaller/greater than a certain threshold
      is sufficiently high
      under input uncertainty.

 \item
      We propose an AL method for IU-rLSE problems.
      Specifically,
      we propose a
      novel AF which can be interpreted as
      an expected improvement for the IU-rLSE problem.
      Although naive implementation of this AF
      requires huge computational cost,
      we propose a computational trick to reasonably approximate the the expected improvement.

 \item
      We show the advantage of the proposed IU-rLSE method both theoretically and empirically.
      Under reasonable assumptions,
      we analyze the accuracy and the convergence of the IU-rLSE method,
      and show that
      it has desirable properties.
      Furthermore,
      we demonstrate the effectiveness of the IU-rLSE method
      by performing numerical experiments both on synthetic and real data.
\end{itemize}

%% file: section2.tex
\section{Preliminaries}
Let $f: \D \rightarrow \R$ be a black-box function whose function values are expensive to evaluate,
where
$\D$
is a compact subset of $\R^d$.
For each input $\bm{x} \in \D$,
assume that a function value is observed as
$y= f({\bm{x}}) + \varrho$,
where
$\varrho \sim \mathcal{N}(0, \sigma^2)$
is an independent Gaussian noise.
Let
$\mathcal{X}$
be a set of finite points in
$\D$.
Given a threshold
$h \in \R$,
the goal of
ordinary
\emph{level set estimation (LSE)} problem~\cite{gotovos2013active}
is to identify
the set of points
$\bm x \in \mathcal{X}$
such that
$f(\bm x) \le h$.

In this paper,
we consider LSE problems under input uncertainty,
which we call
\emph{Input Uncertain Reliable LSE: IU-rLSE}.
In IU-rLSE problems,
when one aims to evaluate the function
$f$
at an input point
$\bm x \in \mathcal{X}$,
one cannot actually observe
$f(\bm x)$,
but observe the function value
$f(\bm s)$
for slightly different input point
$\bm s \in \D$
where
$\bm s$
is a realization of a random variable
$\bm S(\bm x)$
whose density function is written as
$g(\bm s \mid \bm \theta_{\bm x})$.
We first assume that the density function
$g(\bm s \mid \cdot)$
and
the parameters
$\bm \theta_{\bm x}$
are both known,
but later consider the case
where
$\bm \theta_{\bm x}$
is unknown.
The goal of IU-rLSE problems is to
identify a set of points
$\bm x \in \mathcal{X}$
such that
the probability
$
{\mathbb{P}}_{\bm s \sim g(\bm s \mid \bm \theta_{\bm x})}(f(\bm s) \le h)
$
is sufficiently high.
Specifically,
for
each
$\bm x \in \mathcal{X}$
the above probability is written as
\begin{eqnarray*}
 p_{\bm{x}}^* = \int_{f(\bm{s}) \le h} g(\bm{s} \mid \bm{\theta}_{\bm{x}})d\bm{s} = \int_\D \1[f(\bm{s}) < h]g(\bm{s} \mid \bm{\theta}_{\bm{x}})d\bm{s}.
\label{eq:p_star}
\end{eqnarray*}
For a given probability threshold
$\alpha \in (0, 1)$,
we define an upper set $\mathcal{H}$ and a lower set $\mathcal{L}$ on  a subset $\mathcal{X}$ of $\D$ as
\begin{align*}
\mathcal{H}=\{\bm{x} \in \mathcal{X} \mid p_{\bm{x}}^* > \alpha\},\ \mathcal{L}=\{\bm{x} \in \mathcal{X} \mid p_{\bm{x}}^* \leq \alpha\}.
\end{align*}
The goal of IU-rLSE problem is to identify $\mathcal{H}$ with as few function evaluations as possible.
Figure \ref{fig:input_dist} illustrate the basic idea of \emph{reliable} input region.

\begin{figure}[t]
    \begin{center}
        \includegraphics[scale=0.21]{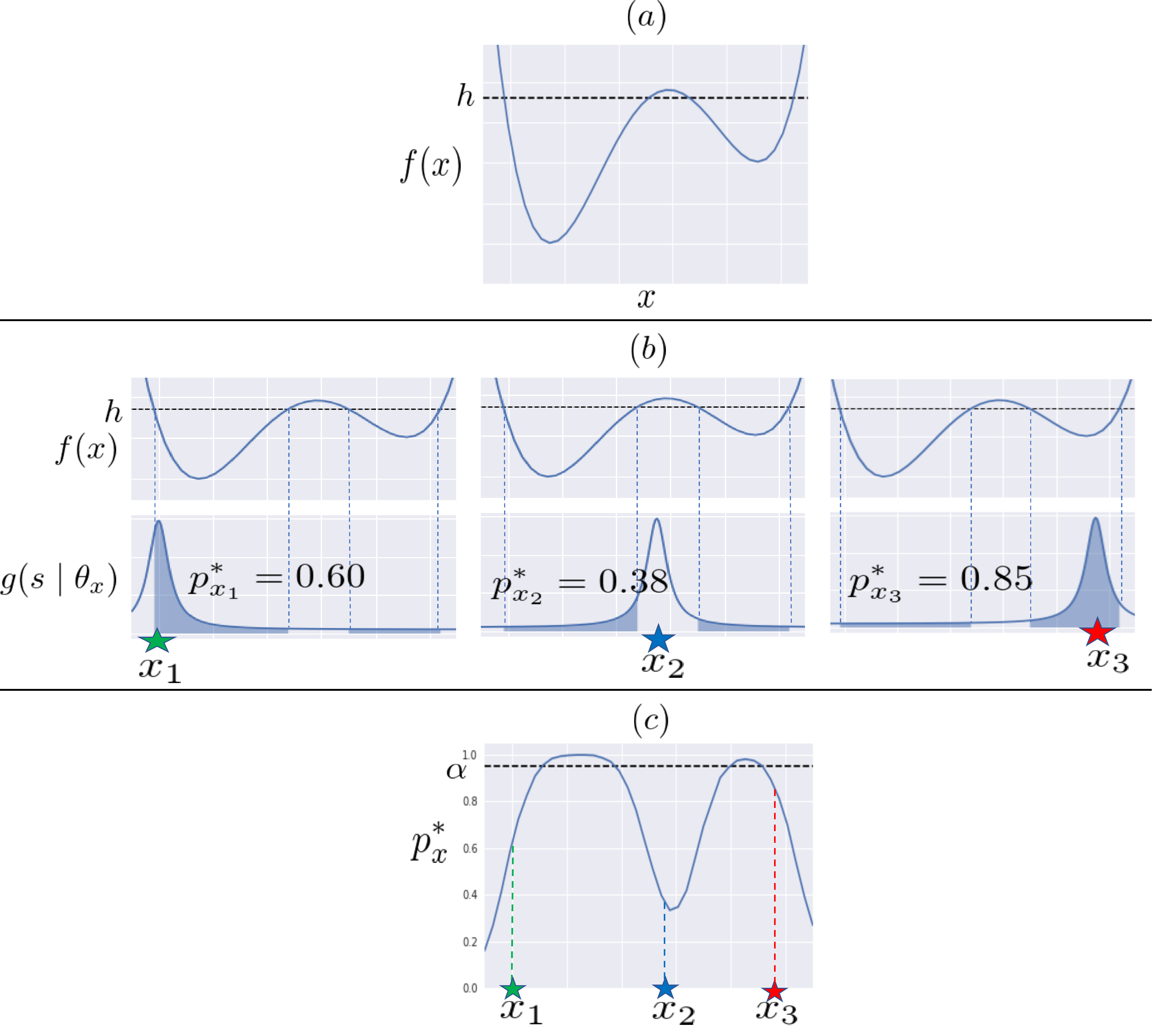}
    \caption{
     An illustrative example of \emph{reliable} input region.
     {\bf a}) The oracle black-box function.
     {\bf b}) Three examples of input points and their uncertainties.
     At each input point, the reliability $p^*_{x_\cdot}$ is defined as the probability of observing outputs smaller than the threshold $h$ when the input uncertainty is taken into account.
     {\bf c}) The reliable input region with reliability threshold $\alpha$ is defined as the subset of the input region in which the reliability $p^*_{x}$ is greater than $\alpha$ (e.g., $\alpha = 0.95$).
     The goal of IU-rLSE problem is to identify the reliable input region as few function evaluations as possible.
     }
    \label{fig:input_dist}
    \end{center}
\end{figure}

%

\subsection{Gaussian Process}
In this paper, to model the unknown function $ f $, we assume  Gaussian process (GP):$\mathcal{GP}(0,k(\bm{s},\bm{s}^{\prime}))$ as a prior distribution of $f$, where $k(\bm{s},\bm{s}^{\prime}):\D \times \D \rightarrow \R$ is a positive definite kernel.
Thus, for any finite points
$\bm{s}_1,\ \dots,\ \bm{s}_t$,
a joint distribution of
   its function values $f_t(\bm{s}_1),...,f_t(\bm{s}_t)$ is defined as
 $(f_t(\bm{s}_1),\ldots,f_t(\bm{s}_t))^{\top} \sim \mathcal{N}_t(\bm{\mu}_t ,\bm{K}_t)$,
where
  $\mathcal{N}_t (\bm{\mu}_t ,\bm{K}_t)$
is a $t$-dimensional normal distribution with mean vector $\bm{\mu}_t = (0,\ldots,0)^{\top} \equiv\bm{0}_t$ and covariance matrix $\bm{K}_t$ whose $(i,j)$th element is $k(\bm{s}_i,\bm{s}_j)$.
From properties of GP, the posterior distribution of $f$ after adding the current data  $\{(\bm{s}_j(\bm{x}_j),\ y_j\}_{j=1}^t$ is also  GP.
 Then, a  mean, variance and covariance of the posterior are respectively given by
\begin{eqnarray*}
\mu_t(\bm{x}) &=& k_t(\bm{x})^{\top}\bm{C}_t^{-1}\bm{y}_t, \\
\sigma_{t}^{2}(\boldsymbol{x}) &=& k_{t}(\boldsymbol{x}, \boldsymbol{x}), \\
k_{t}\left(\boldsymbol{x}, \boldsymbol{x}^{\prime}\right) &=& k\left(\boldsymbol{x}, \boldsymbol{x}^{\prime}\right)-\boldsymbol{k}_{t}(\boldsymbol{x})^{\top} \boldsymbol{C}_{t}^{-1} \boldsymbol{k}_{t}\left(\boldsymbol{x}^{\prime}\right),
\end{eqnarray*}
where $\boldsymbol{k}_{t}(\boldsymbol{x})=(k\left(\boldsymbol{s}_{1}\left(\boldsymbol{x}_{1}\right), \boldsymbol{x}),\ldots, k\left(\boldsymbol{s}_{t}\left(\boldsymbol{x}_{t}\right), \boldsymbol{x}\right)\right)^{\top},\ \boldsymbol{C}_{t}=\left(\boldsymbol{K}_{t}+\sigma^{2} \boldsymbol{I}_{t}\right),\ \boldsymbol{y}_{t}=\left(y_{1}, \ldots, y_{t}\right)^{\top}$ and
$\bm{I}_t$ is a $t$-dimensional identity matrix.

%% file: section3.tex
\section{Proposed Method}
In this section, we propose an efficient active learning method for IU-rLSE.
First of all, we explain the difference between ordinary LSE and  IU-rLSE.
Figure \ref{fig:esth_image} shows a conceptual diagram comparing LSE and IU-rLSE.
In LSE, the purpose is to classify values of the function $ f $.
On the other hand, the purpose of IU-rLSE is to classify probabilities that $ f $ falls below the threshold $h$ under input uncertainty.
%
In ordinary LSE, $ f $ is modeled by GP and classified using a credible interval of $ f ({\bm{x}}) $ \cite{bryan2006active,gotovos2013active}.
On the other hand, the classification target in our setting is the probability $ p_x ^ * $,
 so it is inappropriate to assume GP as in previous studies.
%
Furthermore, acquisition functions such as Straddle \cite{bryan2006active}, LSE \cite{gotovos2013active} and MILE \cite{zanette2018robust} proposed in previous studies can not be used directly in our setting.
In the following subsections, we propose a modeling method for $ p_x ^ * $ and an efficient acquisition function.
\subsection{Estimation of $\mathcal{H}$ and IU-rLSE}
In this subsection, we propose an estimation method of $\mathcal{H}$.
The basic idea is to construct a credible interval $ Q_t (\bm{x}) $ for $ p_x ^ * $ and perform classification based on it.

First, we assume GP as the prior distribution of $f$.
Then, for each $\bm{x} \in \mathcal{X}$, we define the random variable $p_{t, \bm{x}}$ which takes a value in the interval $[0,1]$ as
\begin{align*}
p_{t, \bm{x}} = \int_{\D}\1[f_t(\bm{s}) < h]g(\bm{s} \mid \bm{\theta}_{\bm{x}})d\bm{s}.
\end{align*}
Next, for any $\beta$ with  $\beta^{\frac{1}{2}} \geq 0$, we define the credible interval $Q_t(\bm{x}) = [l_t^{(p)},\ u_t^{(p)}]$ of  $p_{\bm{x}}^*$ as
\begin{align*}
Q_t(\bm{x}) &= [\mu_t^{(p)}(\bm{x}) - \beta^{\frac{1}{2}}\gamma_t(\bm{x}),\ \mu_t^{(p)}(\bm{x}) + \beta^{\frac{1}{2}}\gamma_t(\bm{x})] \\
            &\equiv [l_t^{(p)},\ u_t^{(p)}],
\end{align*}
where $\mu_t^{(p)}(\bm{x}) $ and $\gamma_t^2(\bm{x})$ are given by
\begin{align}
\label{eq:def-mutp}
\mu_t^{(p)}(\bm{x}) &= \E[p_{t, \bm{x}}] = \int_{\D} \Phi_{\bm{s}}g(\bm{s} \mid \bm{\theta}_{\bm{x}}) d\bm{s} \\
\label{eq:def-gammat2}
\gamma_t^2(\bm{x}) &=
 \int_{\D} \V[\1[f_t(\bm{s}) < h]]g(\bm{s} \mid \bm{\theta}_{\bm{x}}) d\bm{s}  \\
&= \int_{\D} \Phi_{\bm{s}}\left(1 - \Phi_{\bm{s}}\right)g(\bm{s} \mid \bm{\theta}_{\bm{x}}) d\bm{s}, \nonumber
\end{align}
and we use the notation $\Phi_{\bm{s}} = \Phi\left(\frac{h - \mu_t{(\bm{s})}}{\sigma_t(\bm{s})}\right)$.
 Here, $\Phi (\cdot)$ is the cumulative distribution function of standard normal distribution.
By using the interval $Q_t(\bm{x})$, we define respectively estimated sets $\mathcal{H}_t$ and $\mathcal{L}_t$ of $\mathcal{H} $ and  $\mathcal{L}$ at the $t$th trial as
\begin{align}
\label{eq:est_high_set}
\mathcal{H}_t &= \{\bm{x} \in \mathcal{X}\ |\ l_t^{(p)} > \alpha - \epsilon\}, \\
\label{eq:est_low_set}
\mathcal{L}_t &= \{\bm{x} \in \mathcal{X}\ |\ u_t^{(p)} \leq \alpha + \epsilon\}.
\end{align}
Moreover, we define the unclassified set $\mathcal{U}_t = \mathcal{X} \backslash (\mathcal{H}_t \cup \mathcal{L}_t)$.

Then, for the credible interval $Q_t(\bm{x})$, the following lemma holds (the proof is given in Appendix \ref{lem:int_process},  \ref{lem:chebyshev}):
\begin{lemma}\label{lem:chebyshev}
    Let $\delta \in (0, 1)$.
    Then, with probability at least $1 - \delta$, it holds that
    \begin{eqnarray*}
        |p_{t, \bm{x}} - \mu_t^{(p)}(\bm{x})| < \delta^{-\frac{1}{2}} \gamma_t(\bm{x})
    \end{eqnarray*}
   where $ \mu_t^{(p)}(\bm{x})$ and $\gamma^2_t(\bm{x})$ are given by \eqref{eq:def-mutp} and \eqref{eq:def-gammat2}, respectively.
\end{lemma}

\begin{figure*}[t]
    \includegraphics[scale=0.35]{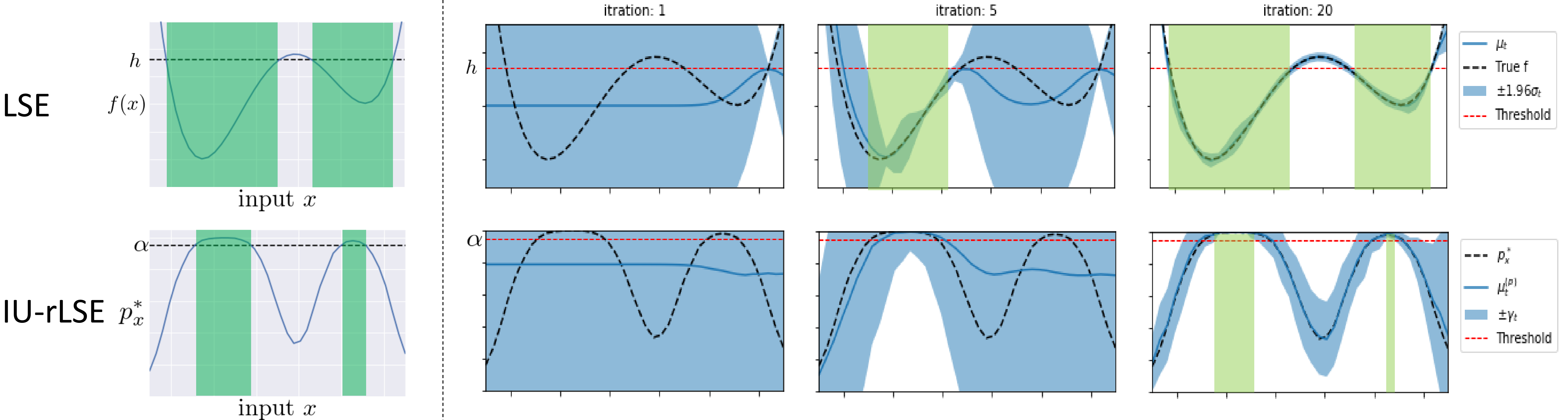}
    \caption{Comparison of LSE and IU-rLSE procedures.
LSE identifies points where the function $ f $ is below the threshold $h$, but IU-rLSE identifies points where the probability $ p_x ^ * $ introduced by input uncertainty is above the threshold $ \alpha $.
As a result, classified points (green area)  by IU-rLSE  differ from ordinary LSE due to input uncertainty.
Moreover, from figures on the right in the upper row, in ordinary LSE, $ f $ is modeled by GP, and classification is performed based on credible  intervals of $ f $.
  On the other hand, in IU-rLSE, it is necessary to construct  credible intervals of $ p_x ^ * $ appropriately.}
    \label{fig:esth_image}
\end{figure*}

\subsection{Acquisition function}
In this subsection, we propose an acquisition function to determine a next evaluation point.
%
Our proposed AF is based on the Maximum Improvement for Level-set Estimation  (MILE) introduced by \cite{zanette2018robust}.
In MILE, the point that maximizes the expected   classification improvement after adding one point is taken as the next evaluation point.
However, MILE can not be directly applied under input uncertainty.
Therefore, we extend MILE to the setting in this paper, and propose  rational approximations.
In addition, by combining the proposed AF with random sampling, we show that   our proposed algorithm converges with   probability 1.

\subsubsection{AF based on expected classification improvement and its approximation}
Let $\bm{s}^*$ be an entered point, and let $y^* = f(\bm{s}^*) + \varrho$ be an observed value corresponding to $\bm{s}^*$.
Moreover, let $\mathcal{H}_t(\bm{s}^*,\ y^*)$ denote an estimated set of $\mathcal{H}$  when $(\bm{s}^*, y^*)$ is added.
 Then, the expected classification improvement $a_t(\bm{x})$ when considering input uncertainty for the point $\bm{x} \in \mathcal{X}$ is given by
\begin{equation}
a_{t}\left(\boldsymbol{x}\right)=\int_{\D}\mathrm{E}_{y^{*}}\left[\left|\mathcal{H}_{t}\left(\boldsymbol{s}^{*}, y^{*}\right)\right|-|\mathcal{H}_{t}| \right] g\left(\boldsymbol{s}^{*} | \theta_{\boldsymbol{x}}\right) d \boldsymbol{s}^{*},
\label{eq:org_aq}
\end{equation}
where the expected value in
(\ref{eq:org_aq}) can be expressed as
\begin{align}
&\quad\mathrm{E}_{y^{*}}\left[\left|\mathcal{H}_{t}\left(\boldsymbol{s}^{*}, y^{*}\right)\right|-|\mathcal{H}_{t}| \right] \nonumber \\
&= \sum_{\bm{x} \in \mathcal{X}}\int \1_{\bm{x}\mid \bm{s}^*, y^*}p(y^* \mid \bm{s}^*) dy^* - |H_{t}|.
\label{eq:aq_exp}
\end{align}
We denotes indicator function $\1[\mu_t^{(p)}(\bm{x}\mid\bm{s}^*, y^*) - \beta^{\frac{1}{2}}\gamma_t^{(p)}(\bm{x}\mid\bm{s}^*, y^*) > \alpha - \epsilon]$
as $\quad\1_{\bm{x}\mid \bm{s}^*, y^*}$.
Here,
  $\mu_t^{(p)}(\bm{x} \mid \bm{s}^*, y^*)$ and $ \gamma_t^2(\bm{x} \mid \bm{s}^*, y^*)$ are given by
\begin{align}
    \label{eq:next_mu}
    \mu_t^{(p)}(\bm{x} \mid \bm{s}^*, y^*)&=\int_{\D} \Phi_{\bm{s}\mid y^*}g(\bm{s} \mid \bm{\theta}_{\bm{x}}) d\bm{s}, \\
    \label{eq:next_gamma}
    \gamma_t^{2}(\bm{x} \mid \bm{s}^*, y^*)&= \nonumber \\
       \int_{\D} \Phi_{\bm{s}\mid y^*}&\left(1 - \Phi_{\bm{s}\mid y^*}\right)g(\bm{s} \mid \bm{\theta}_{\bm{x}}) d\bm{s}
\end{align}
where $p(y^* \mid \bm{s}^*)$ is a density function of  $y^*$  corresponding to $\bm{s}^*$,
and we use the notation $\Phi_{\bm{s}\mid y^*} =\Phi\left(\frac{h - \mu_t{(\bm{s}\mid \bm{s}^*, y^*)}}{\sigma_t(\bm{s} \mid \bm{s}^*)}\right)$.
Furthermore, $\mu_t(\bm{x} \mid \bm{s}^*, y^*)$ and $ \sigma_t^2(\bm{s}\mid\bm{s}^*)$ are a posterior mean and variance of
$f({\bm{s}})$ after adding  $(\bm{s}^*, y^*)$.

Next, we consider the calculation cost of $a_t ({\bm{x}})$.
From \eqref{eq:org_aq}--\eqref{eq:next_gamma}, in order to calculate $a_t ({\bm{x}})$, it is necessary to perform integration three times.
When one integral calculation is approximated by $ M $ times sampling, the calculation cost of $ a_t (\bm {x}) $ is $ O (| \mathcal {X} | M ^ 3) $.
However, since this is not a realistic cost, we propose a reasonable approximation of $ a_t (\bm {x}) $.
%
For this reason, we approximate  (\ref{eq:next_mu}) and (\ref{eq:next_gamma}) as
\begin{align*}
    \mu_t^{(p)}(\bm{x} \mid \bm{s}^*, y^*) &\approx \Phi_{\overline{\bm{s}}}, \\
    \gamma_t^2(\bm{x} \mid \bm{s}^*, y^*) &\approx \Phi_{\overline{\bm{s}}}\left(1 - \Phi_{\overline{\bm{s}}}\right),
\end{align*}
where $\overline{\bm{s}}$ is the expected value of ${\bm{s}}$ with respect to $g(\bm{s} \mid \bm{\theta}_{\bm{x}})$,
and we use the notation $\Phi_{\overline{\bm{s}}} = \Phi\left(\frac{h - \mu_t{(\overline{\bm{s}}\mid\bm{s}^*, y^*)}}{\sigma_t(\overline{\bm{s}}\mid\bm{s}^*)}\right)$.
%
Hence,    (\ref{eq:aq_exp}) can be approximated as
\begin{align}
    &\quad \mathrm{E}_{y^{*}}\left[\left|\mathcal{H}_{t}\left(\boldsymbol{s}^{*}, y^{*}\right)\right|-|\mathcal{H}_{t}| \right] \nonumber \\
    &= \sum_{\bm{x} \in \mathcal{X}}\int \1_{\bm{x}\mid \bm{s}^*, y^*}\ p(y^* \mid \bm{s}^*) dy^* - |\mathcal{H}_{t}| \nonumber \\
    &\approx\sum_{\overline{\bm{s}} \in \overline{\mathcal{S}}}\int \1_{\bm{x}\mid \bm{s}^*, y^*}\ p(y^* \mid \bm{s}^*) dy^* - |\mathcal{H}_{t}|,
    \label{eq:approx_exp}
\end{align}
where $\overline{\mathcal{S}} = \{\E_{g(\bm{s} \mid \bm{\theta}_{\bm{x}})}[\bm{s} \mid \bm{x}] \mid \bm{x} \in \mathcal{X}\}$.
Moreover, the inequality in the indicator function in (\ref{eq:approx_exp}) can be written as follows (details are given in
Appendix\ref{lem:second_inequity}:
\begin{equation*}
 c < \Phi_{\overline{\bm{s}}} \leq 1,\
\end{equation*}
\begin{equation*}
    c = \frac{2(\alpha - \epsilon) + \beta + \sqrt{\beta^2 + 4(\alpha - \epsilon)\beta - 4(\alpha- \epsilon)^2\beta}}{2(1 + \beta)}.\
\end{equation*}
Therefore, the following holds:
\begin{align*}
     &\hspace{10pt}\frac{h - \mu_t{(\overline{\bm{s}}\mid\bm{s}^*, y^*)}}{\sigma_t(\overline{\bm{s}}\mid\bm{s}^*)} < \Phi^{-1}(c) \\
     &\Leftrightarrow \mu_t{(\overline{\bm{s}}\mid\bm{s}^*, y^*)} > h - \sigma_t(\overline{\bm{s}}\mid\bm{s}^*)\Phi^{-1}(c).
\end{align*}
Moreover,
 the posterior mean $\mu_t{(\overline{\bm{s}}\mid\bm{s}^*, y^*)}$ can be written as follows (see, e.g., \cite{gpml}):
\begin{equation*}
\mu_t{(\overline{\bm{s}}\mid\bm{s}^*, y^*)} = \mu_t(\overline{\bm{s}}) - \frac{k_t(\overline{\bm{s}},\ \bm{s}^*)}{\sigma_t^2(\bm{s}^*) + \sigma^2}(y^* - \mu_t(\bm{s}^*)).
\end{equation*}
Thus, noting that  $\mu_t{(\overline{\bm{s}}\mid\bm{s}^*, y^*)}$ can be expressed as the linear function of $y^*$,
the inequality in the indicator function in (\ref{eq:approx_exp}) can be also written as the linear function of  $y^*$.
Hence, by using the cdf of standard normal distribution,
the integral in (\ref{eq:approx_exp}) can be solved analytically because
 $p(y^* \mid \bm{s}^*)$ is a density function of normal distribution  (details are given in Appendix \ref{lem:drive_aq}.

From the above discussion, we propose the following approximate AF $\hat{a}_t(\bm{x})$:
\begin{align}
&\hat{a}_t(\bm{x}) = \int_{\D}\Biggl\{\sum_{\overline{\bm{s}} \in \overline{\mathcal{S}}} \Phi\Biggl(\frac{\sqrt{\sigma_{t}^{2}\left(\bm{s}^{*}\right)+\sigma^{2}}}{\left|k_{t}\left(\overline{\bm{s}},
\boldsymbol{s}^{*}\right)\right|} (\mu_{t}(\overline{\bm{s}}) \nonumber \\
&-\Phi^{-1}(c) \sigma_{t}\left(\overline{\bm{s}} | \boldsymbol{s}^{*}\right)-h)\Biggr)
  - |\mathcal{H}_t| \Biggr\} g(\bm{s}^* \mid \bm{\theta_{\bm{x}}}) d\bm{s}^*,
\label{eq:proposed_aq}
\end{align}
where
\begin{equation*}
\overline{\mathcal{S}}=\{\E_{g(\bm{s} \mid \bm{\theta}_{\bm{x}})}[\bm{s} \mid \bm{x}] \mid \bm{x} \in \mathcal{X}\}.
\end{equation*}
Since \eqref{eq:proposed_aq} has only one integral, the calculation cost of  \eqref{eq:proposed_aq} is $O(|\mathcal{X}|M)$.
However, approximation accuracy of $\hat{a}_t(\bm{x})$ is not necessary good because
$\hat{a}_t(\bm{x})$ considers only the classification of $\mathcal{\overline{S}}$.
As the IU-rLSE progresses and posterior variances of $f$ corresponding to points in  $ \mathcal {\overline{S}} $ is reduced sufficiently,  all points in  $ \mathcal {\overline {S}} $ are classified. As a result, it is expected that $\hat{a}_t(\bm{x})$ will not work well after this.
To avoid this problem, we consider adaptively determining  $ \mathcal {\overline {S}} $ for each trial.
For each trial $t$, we define $ \mathcal {\overline {S}} _t$ as
\begin{align}
\mathcal {\overline {S}} _t = \left \{ \tilde{\bm{s}} _{ {\bm{x}}} \equiv
\argmax_{\bm{s} \in \D} \Phi_{\bm{s}\mid y^*}\left(1 - \Phi_{\bm{s}\mid y^*}\right)g(\bm{s} \mid \bm{\theta}_{\bm{x}}) \right. \nonumber \\
\Biggl | \  {\bm{x}} \in \mathcal{X}
\Biggr \}. \label{eq:ada_strategy}
\end{align}
Note that $ \tilde{\bm{s}} _{ {\bm{x}}}$ is the  point which maximizes the integrand in   $\gamma_t^2(\bm{x} \mid \bm{s}^*, y^*)$.
%
The pseudo code of our proposed method is shown in Algorithm \ref{alg:proposed}.
In the proposed method, for each trial $t$, with probability $1-p_t$, we select $ \bm {x} \in \mathcal {X} $ based on $\hat{a}(\bm{x})$, and otherwise uniformly select $ \bm {x} \in \mathcal {X} $.
Here, $\mathcal{B}(p_t)$ in Algorithm \ref{alg:proposed} is Bernoulli distribution with parameter $p_t$.
%
\begin{algorithm}[t]
    \caption{Proposed LSE}
    \label{alg:proposed}
    \begin{algorithmic}
        \REQUIRE Initial training data, GP prior $\mathcal{GP}(0,\ k(\bm{x}, \bm{x}^{\prime}))$,\ probabilities $\{p_t\}_{t \in \N}$
        \ENSURE Estimated sets $\hat{\mathcal{H}}$, $\hat{\mathcal{L}}$
        \STATE $\hat{\mathcal{H}}_0 \leftarrow \emptyset,\ \hat{\mathcal{L}}_0 \leftarrow \emptyset,\ \hat{\mathcal{U}}_0 \leftarrow \mathcal{X}$
        \STATE $t \leftarrow 1$
        \WHILE{$\hat{\mathcal{U}}_{t-1} \neq \emptyset$}
            \STATE $\hat{\mathcal{H}}_t \leftarrow \hat{\mathcal{H}}_{t-1},\ \hat{\mathcal{L}}_t \leftarrow \hat{\mathcal{L}}_{t-1},\ \hat{\mathcal{U}}_t \leftarrow \hat{\mathcal{U}}_{t-1}$
            \FORALL{$\bm{x} \in \mathcal{X}$}
                \STATE Compute credible interval $Q_t(\bm{x})$ from GP
            \ENDFOR
            \STATE Compute $\mathcal{H}_t, \mathcal{L}_t$ and $\mathcal{U}_t$ from (\ref{eq:est_high_set}), (\ref{eq:est_low_set}) and generate $r_t$ from $\mathcal{B}(p_t)$
            \IF{$r_t = 0$}
                \STATE Compute $\mathcal{\overline{S}}_t$ from (\ref{eq:ada_strategy})
                \STATE $\bm{x}_t = \argmax_{\bm{x} \in \mathcal{X}} \hat{a_t}(\bm{x})$
            \ELSE
                \STATE Select $\bm{x}_t$ at random
            \ENDIF
            \STATE Generate $\bm{s}_t(\bm{x})$ from $\bm{S}(\bm{x}_t)$
            \STATE $y_t \leftarrow f(\bm{s}_t(\bm{x_t})) + \varepsilon_t$
            \STATE $t \leftarrow t + 1$
        \ENDWHILE
        \STATE $\hat{\mathcal{H}} \leftarrow \hat{\mathcal{H}}_{t-1}, \hat{\mathcal{L}} \leftarrow \hat{\mathcal{L}}_{t-1}$
    \end{algorithmic}
\end{algorithm}

\subsubsection{Unknown input distribution}
In this subsection, we consider the case that the density function $g(\bm{s} \mid \bm{\theta}_{\bm{x}})$ is unknown.
%
In this case, it is necessary to estimate it during trials.
One natural approach is to assume certain function  form for
$g(\bm{s} \mid \bm{\theta}_{\bm{x}})$ and estimate unknown parameters $\bm{\theta}_{\bm{x}}$.
 Nonetheless, parameter estimation is still difficult if we assume
  a different $\bm{\theta}_{\bm{x}}$ for each point $\bm{x} \in \mathcal{X}$.
For this reason, we assume that $\bm{\theta}_{\bm{x}} $ can be separated as
 $\bm{\theta}_{\bm{x}} = (\bm{\hat{\theta}}_{\bm{x}}, \bm{\xi})$, where
$\bm{\hat{\theta}}_{\bm{x}}$ and  $\bm{\xi}$ are respectively known and unknown parameters.
Then, assuming a prior distribution $\pi(\bm{\xi})$ for $\bm{\xi}$,
$g(\bm{s}\mid \bm{\theta}_{\bm{x}})$ can be estimated using a posterior distribution $\pi_t(\bm{\xi})$
 after data observation as follows:
\begin{equation}
g_t(\bm{s} \mid \bm{\theta}_{\bm{x}}) = \int g(\bm{s} \mid \bm{\theta}_{\bm{x}}) \pi_t(\bm{\xi}) d\bm{\xi}.
\label{eq:approx_g}
\end{equation}
Therefore, based on (\ref{eq:approx_g}), we can compute (\ref{eq:est_low_set}), (\ref{eq:est_low_set}),\ (\ref{eq:proposed_aq}), and (\ref{eq:ada_strategy}).

%% file: section4.tex
\section{Theoritical Result}\label{sec:theory}
In this section,
we present two theorems for accuracy and convergence.
First, for each point $\bm{x} \in \mathcal{X}$,
we define the misclassification loss   $e_\alpha ({\bm{x}})$ as
\begin{eqnarray*}
e_\alpha ({\bm{x}}) =
\left \{
\begin{array}{ll}
\max \{ 0,  p^\ast_{\bm{x}} -\alpha \}  & \text{if} \  \bm{x} \in \hat{L}  \\
\max \{ 0, \alpha- p^\ast_{\bm{x}}  \} & \text{if} \ \bm{x} \in \hat{H}
\end{array}
\right . .
\end{eqnarray*}
Then, the following theorem holds for classification accuracy:
\begin{theorem}\label{thm:seido}
For any $\alpha \in (0,1)$, $\delta \in (0,1)$ and $\epsilon >0$, if
$\beta^{1/2} = (\delta/|\mathcal{X} | ) ^{-1/2} $,
 with probability at least  $ 1- \delta $, the misclassification loss is less than $ \epsilon $ when the   algorithm is finished.
That is, the following inequality holds:
$$
\PR \left ( \max _{{\bm{x}} \in \mathcal{X}} e_\alpha ({\bm{x}}) \leq \epsilon \right ) \geq 1-\delta.
$$
\end{theorem}
The proof is given in Appendix \ref{sec:appendix_seido}.

The next theorem states the convergence property of the proposed IU-rLSE method.
Unlike ordinary LSE problem,
the coverngence of IU-rLSE is non-trivial
since one cannot evaluate the function at desired input points.
Therefore,
we conduct careful probabilistic analysis on the convergence in the following theorem.
The following theorem gives a probabilistic evaluation for convergence of the algorithm under regular conditions {\sf (A1)--(A4)} (given in Appendix).
\begin{theorem}\label{thm:owaru}
Assume that regular conditions {\sf (A1)--(A4)} hold. Then, for any $ \alpha \in (0,1) $, $ \epsilon> 0 $ and $ \beta> 0 $, with probability 1, the algorithm ends after point evaluations for a finite number of times.
\end{theorem}
The proof is given in Appendix \ref{sec:appendix_owaru}.
%

%% file: section5.tex
\section{Numerical Experiment}\label{sec:exp}
In this section, we compared the performance of existing methods and the proposed method through numerical experiments, and confirmed the effectiveness of the proposed method.
For comparison, we considered existing methods
 Straddle\cite{bryan2006active},\ MILE\cite{zanette2018robust} and random sampling.
On the other hand, we used $\beta^{1/2} = 3$ for calculating $\hat{a}_t ({\bm{x}})$.
Furthermore,  estimation of $\mathcal{H}$ was also performed using $\beta^{1/2} = 3$.
In this experiments, we set $p_t=0$ and $\epsilon=0$ for simplicity.
Moreover, we used $F1$-score as the classification accuracy.
In addition, for each synthetic/real function, we calculated the true probability $p^\ast _{\bm {x} } $ by using 100,000 Monte Carlo simulations and defined the true  $\mathcal{H}$.

\subsection{Synthesic Experiment}
\subsubsection{1d-synthesic function}\label{subsub:1dimf}
We confirmed the  classification accuracy and the goodness of the approximation of AF in IU-rLSE by using the following function $f(x)$:
\begin{eqnarray}
f(x) = 3 - 40x + 38x^2 - 11x^3 + x^4. \label{eq:1dimf}
\end{eqnarray}
In addition, we defined $\mathcal{X}$ as the grid points when $[-0.5, 5.5]$ divided into 40.
Furthermore, we used Gaussian kernel $k({\bm{x}},{\bm{x}}') = \sigma^2_f \exp ( \| {\bm{x}} - {\bm{x}}' \|^2/L)$ and set
$\sigma_f^2 = 100$ and
$L=0.5$.
Moreover, we used $\sigma^2 = 10^{-4}$ as the error variance and $h=8$ as the threshold for $f(x)$.
In this experiment, we considered the following two distributions as the input distribution:
    \begin{description}
        \item[Case1] $\bm{S}(\bm{x}) = \bm{x} + Gamma(5,\ 0.03)$.
        \item[Case2] $\bm{S}(\bm{x}) = \bm{x} + \mathcal{N}(0,\ 0.07^2)$.
    \end{description}
Here, $Gamma (a,b)$ is the gamma distribution with parameters $a$ and $b$.
Experiment results are given in Figure \ref{fig:1d_prec}.
From Figure \ref{fig:1d_prec}, we can confirm that the proposed method has better performance than existing methods.
Note that existing methods Straddle, MILE and RS focus on the classification for $f$.
Recall that our target function is $p^\ast_{ {\bm{x}} }$, not $f$.
Thus, since the classification target in  existing methods is different, it is natural that the accuracy is low.
However, the classification procedure in the proposed method can also be applied to existing methods.
Specifically, in each iteration of IU-rLSE, classification is performed using  \eqref{eq:est_high_set}, \eqref{eq:est_low_set}, and  existing methods are used only for selecting the next evaluation point.
In other words, only the acquisition function of the existing method is used, and the proposed method is used as the classification method.
Hereinafter, this method will be used as the existing method.



\begin{figure}[t]
    \begin{center}
        \includegraphics[scale=0.24]{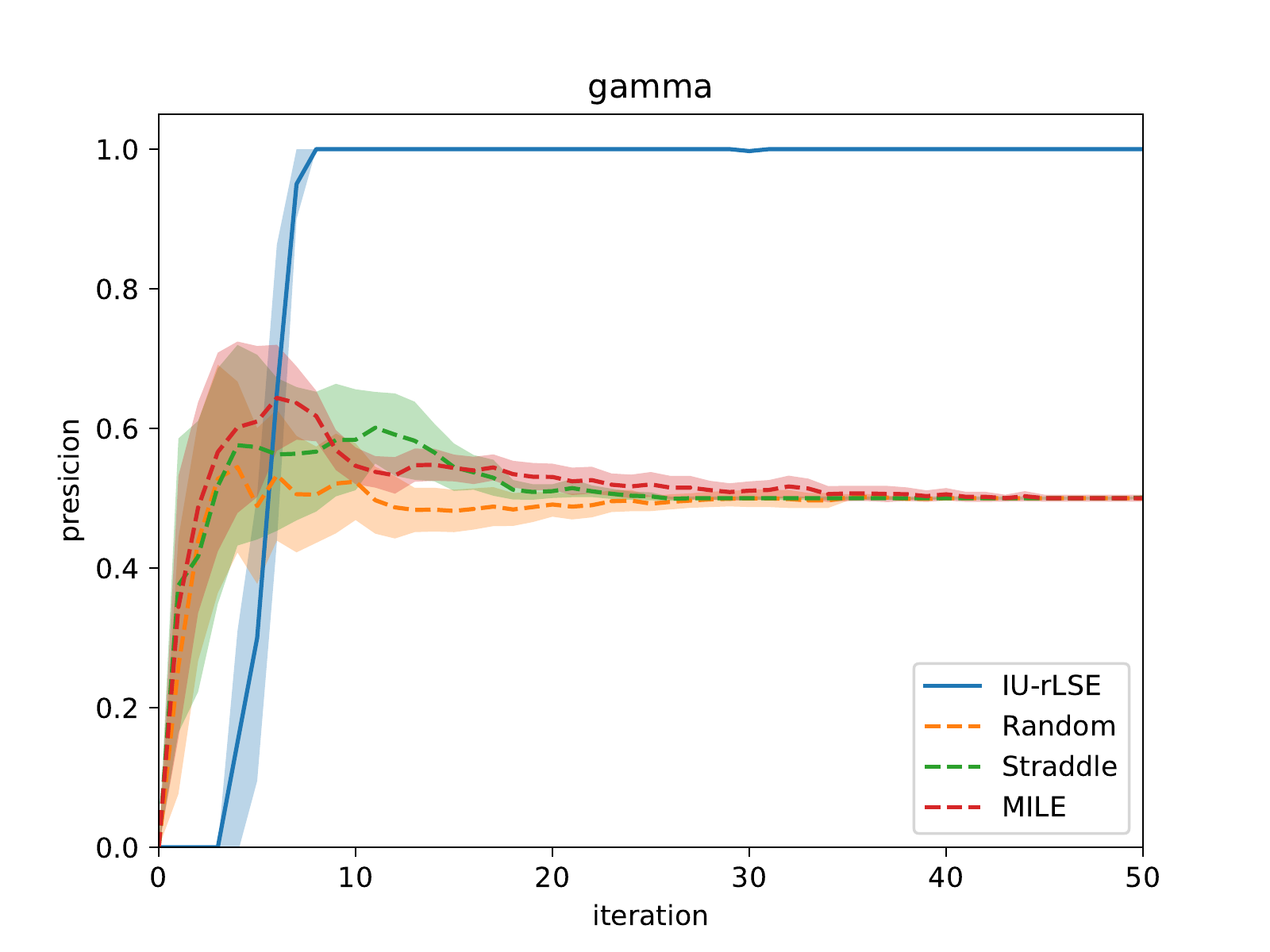}
        \includegraphics[scale=0.24]{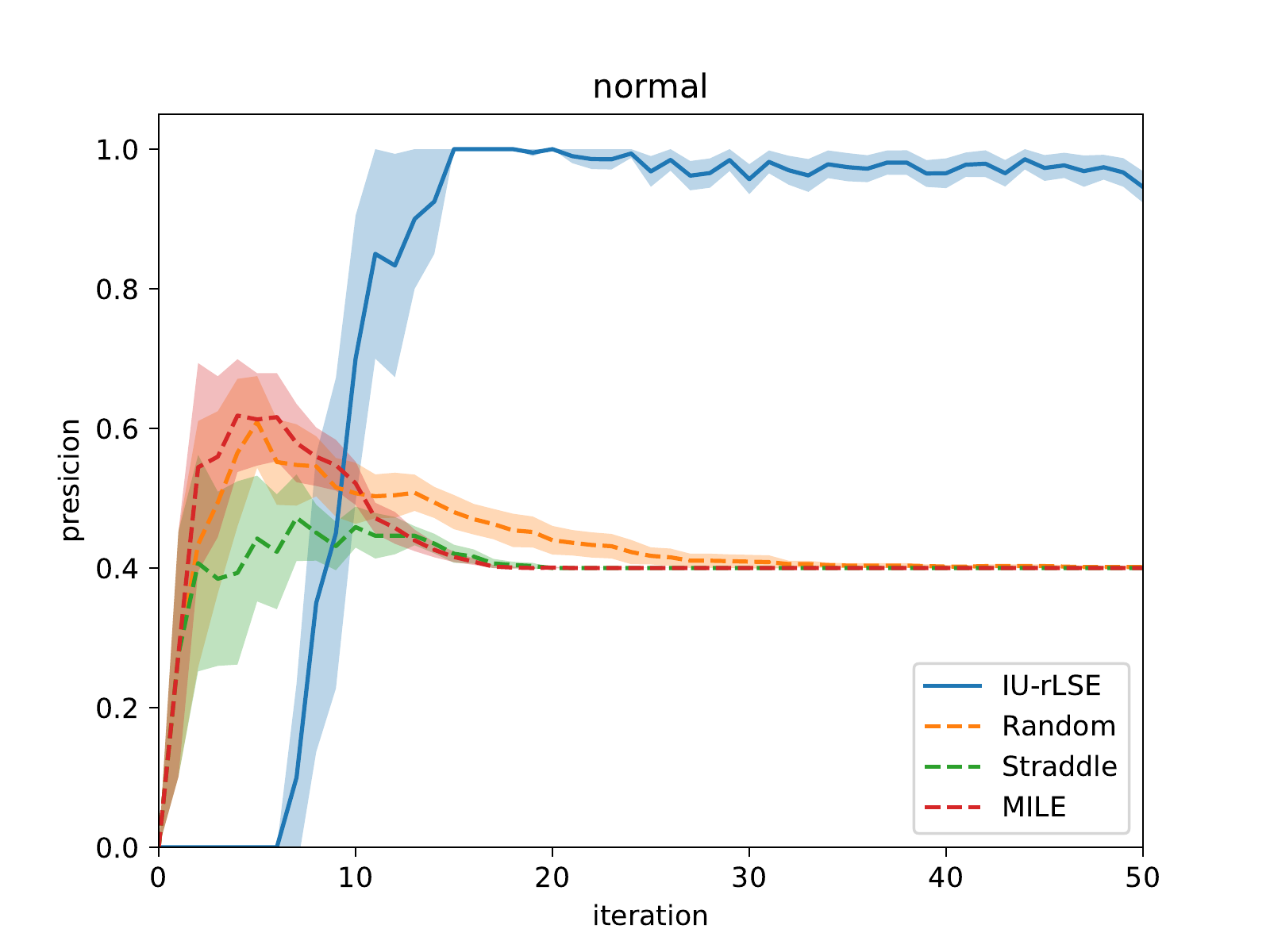}
        \caption{Average accuracy based on 20 Monte Carlo simulations for the one-dimensional synthetic function.
The left and right side figures represent  Case1 and Case2, respectively.
Shaded areas represent confidence intervals for $F1$-score ($ \pm 1.96 \times $ [standard error]).}
        \label{fig:1d_prec}
    \end{center}
\end{figure}
\subsubsection{Sinusoidal function}\label{subsub:sin}
In this subsection, we used $f(x_1, x_2) = -\sin(10x_1) - \cos(4x_2) + \cos(3x_1x_2)$ as the true function.
Here,
in numerical experiments in \cite{zanette2018robust}, $-f(x_1,x_2)$ was used as the true function.
Moreover, we defined $\mathcal{X}$ as the grid points when
  $[0, 1] \times [0, 2]$ divided into $30 \times 60$.
Furthermore, we used the Gaussian kernel with
 $\sigma_f^2 = e^2$ and
$L=2e^{-3}$
In addition, we used
  $\sigma^2 = 10^{-4}$ and $h=-0.5$.
\begin{figure}[t]
    \begin{center}
        \includegraphics[scale=0.26]{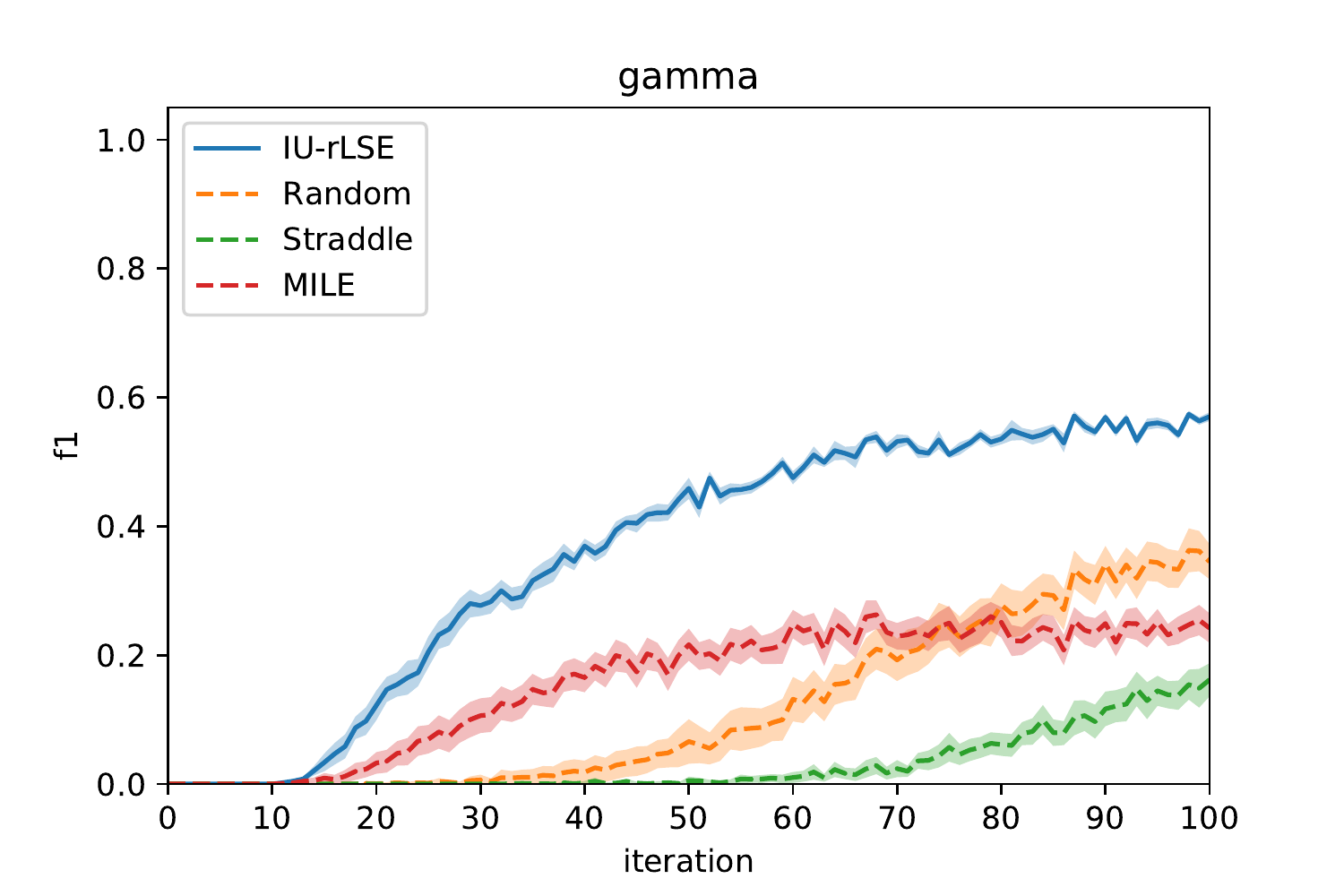}
        \includegraphics[scale=0.26]{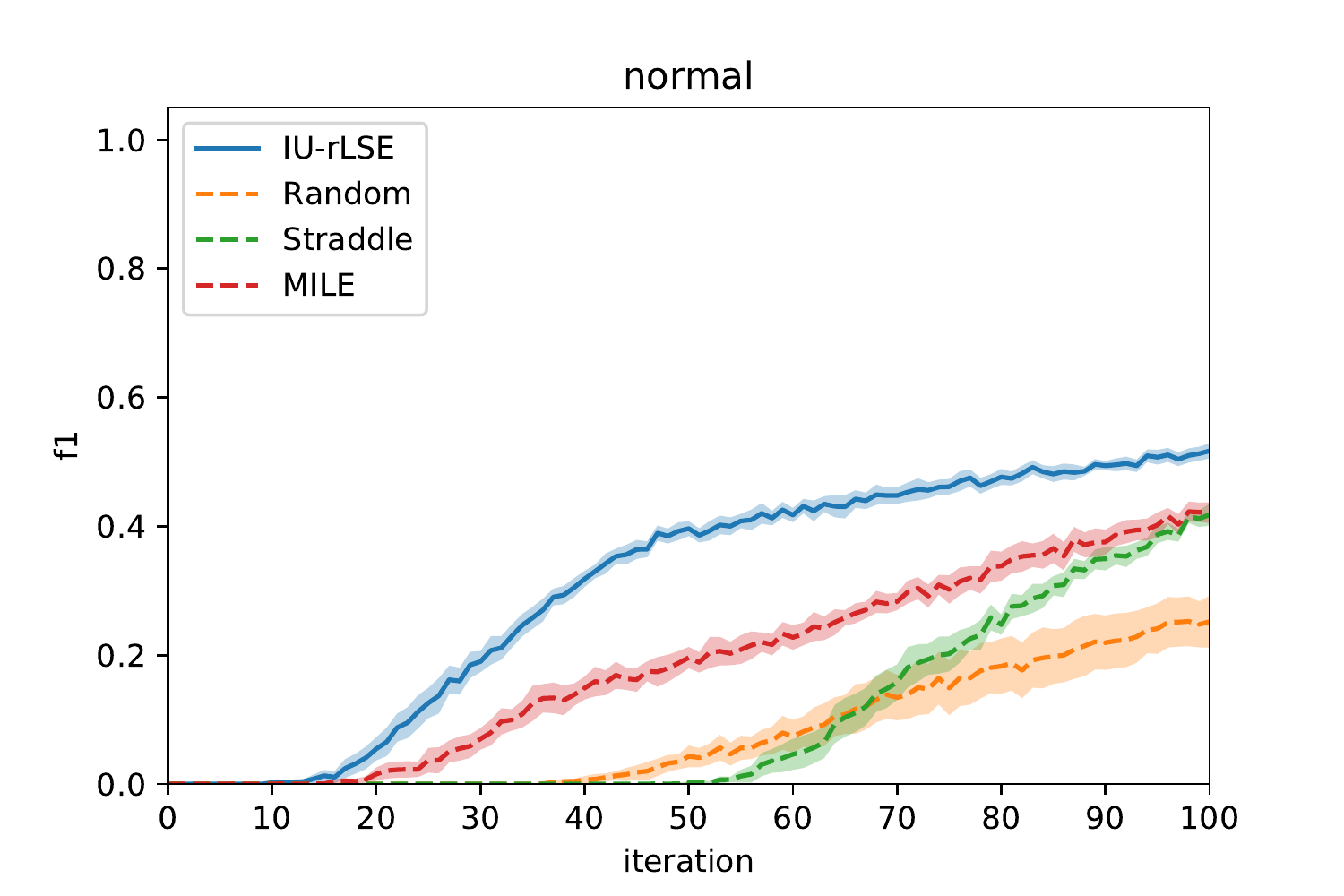}
        \caption{Average $F1$-score based on 20 Monte Carlo simulations for the Sinusoidal function.
The left and right side figures represent  Case1 and Case2, respectively.}
        \label{fig:sinusoidal}
    \end{center}
\end{figure}

In this experiment, we assumed that the input was two dimensional random vector whose elements have same distribution and are mutually independent.
Furthermore, as the distribution of each element, the same setting as in previous subsection was used.
Figure \ref{fig:sinusoidal} shows the experiment result based on 20 Monte Carlo simulations.
From Figure \ref{fig:sinusoidal}, we can confirm that the $F1$-score  based on  the proposed  method is larger  than those of existing methods.

\subsubsection{Himmelblau function}
In this subsection, as the true function, we considered the following Himmelblau function with added $-100$:
\begin{eqnarray*}
f(x_1,\ x_2) = (x_1^2 + x_2 - 11)^2 + (x_1 + y_1^2 - 7)^2 - 100.
\end{eqnarray*}
We defined $\mathcal{X}$ as the grid points when
 $[-5, 5] \times [-5, 5]$ divided into $50 \times 50$.
Moreover, we used Gaussian kernel with
  $\sigma_f^2 = e^8$ and
$L=2$
 Furthermore, we set $\sigma^2 = 10^{-4}$ and $h=0$.

In this experiment, we assumed the following two cases for the input distribution of each element:
\begin{description}
    \item[Case1] $\bm{S}(\bm{x}) = \bm{x} + Gamma(5,\ 0.15)$
    \item[Case2] $\bm{S}(\bm{x}) = \bm{x} + \mathcal{N}(0,\ 0.5^2)$
\end{description}
Figure \ref{fig:himmelblau} shows the experiment result based on 20 Monte Carlo simulations.
Also in this experiment, we can confirm the similar results as in the previous experiments.

\begin{figure}
    \begin{center}
        \includegraphics[scale=0.24]{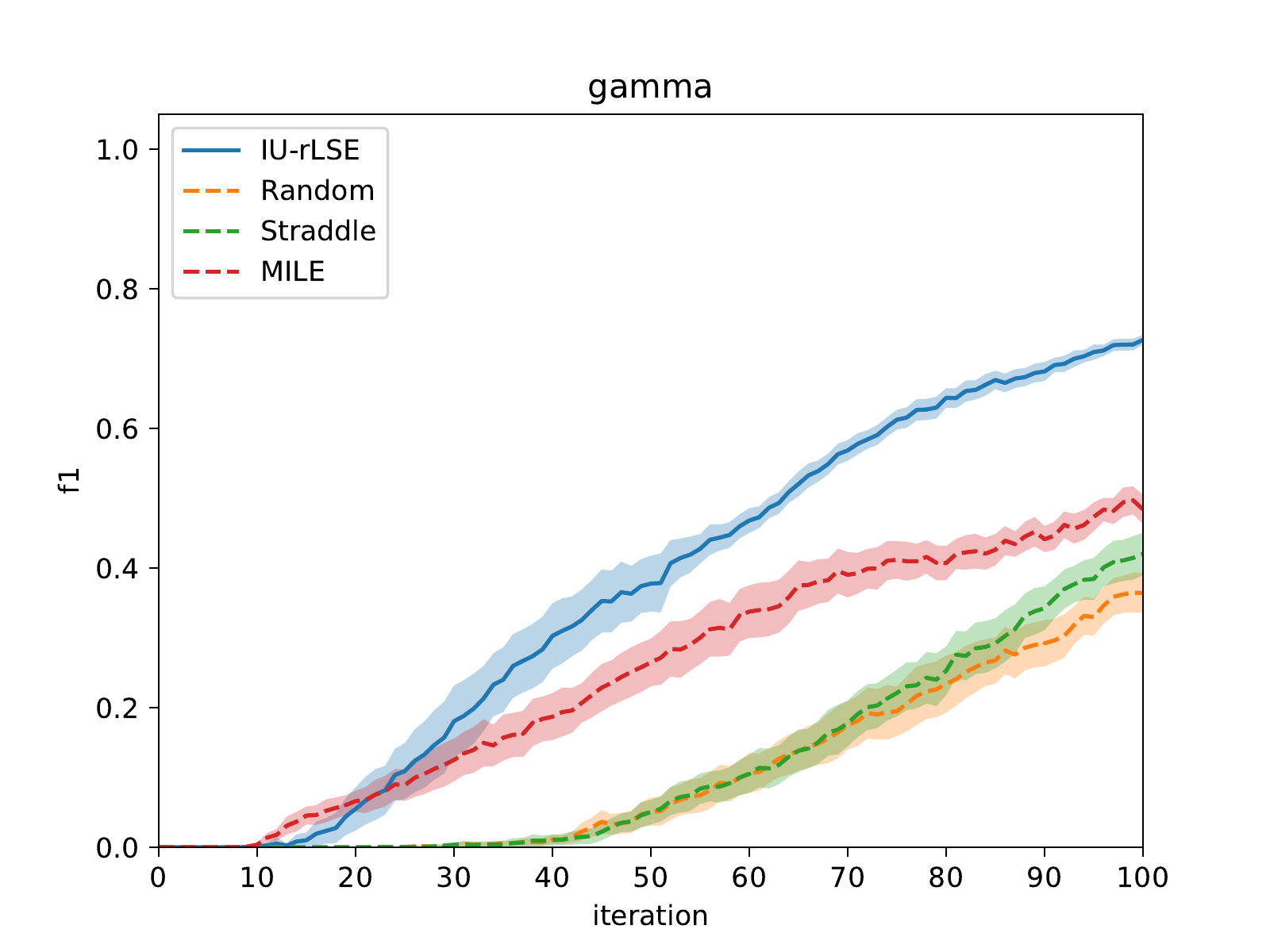}
        \includegraphics[scale=0.24]{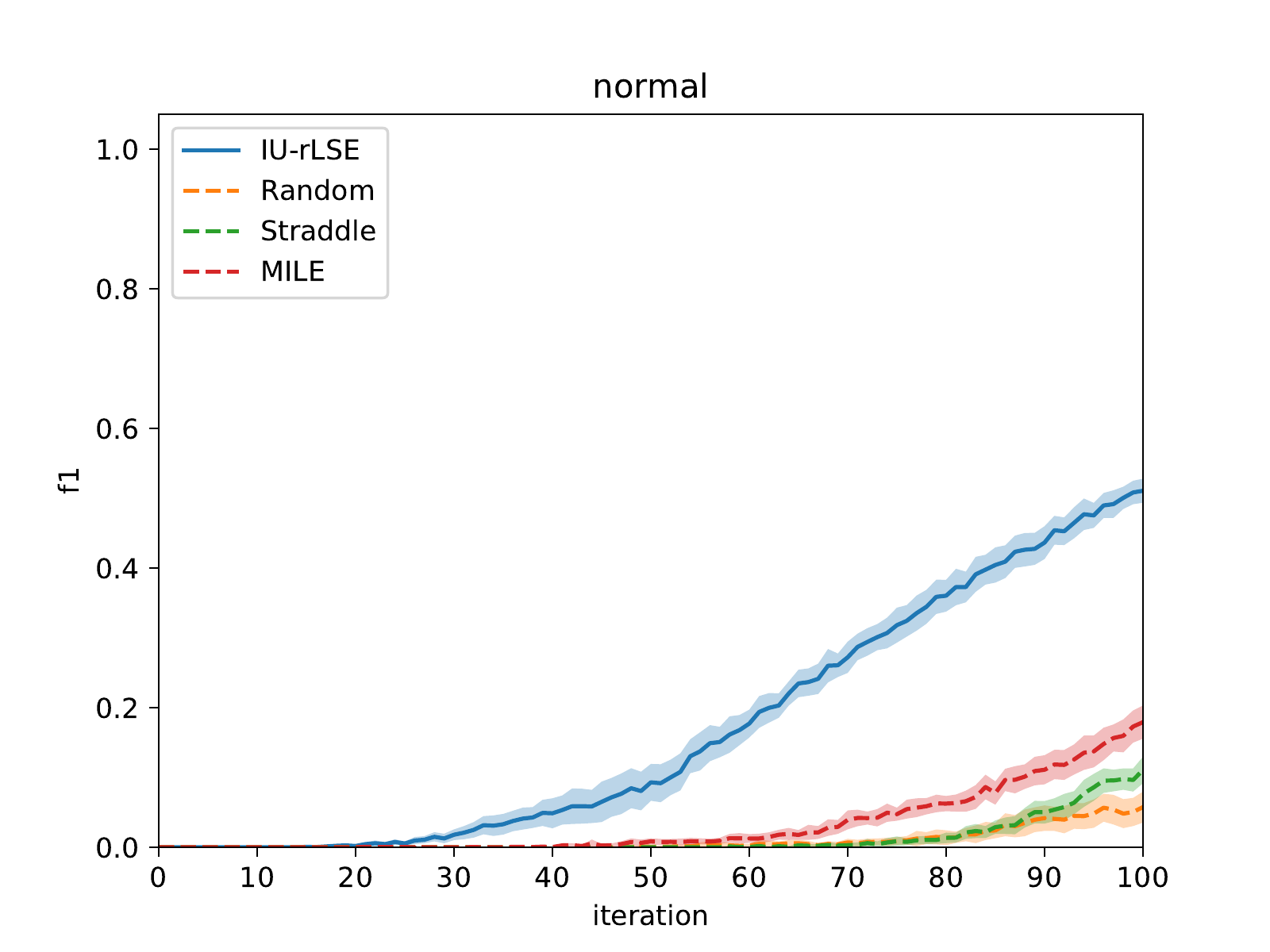}
        \caption{Average $F1$-score based on 20 Monte Carlo simulations for Himmelblau function.
The left and right side figures represent  Case1 and Case2, respectively.}
        \label{fig:himmelblau}
    \end{center}
\end{figure}
\subsubsection{1d-synthesic function with unknown inputs distribution}
In this subsection, we considered the situation that  input distributions are unknown.
We considered the  same setting as in Subsection \ref{subsub:1dimf} except input distributions.
In this experiment, we considered the following input distribution:
\begin{eqnarray*}
    \bm{S}(\bm{x}) = \bm{x} + \mathcal{N}(\hat{\mu},\ \hat{\sigma}^2).
\end{eqnarray*}
Under this setting, we considered the following two cases:
\begin{description}
    \item[Case1] The true parameter is $(\hat{\mu},\ \hat{\sigma}^2) = (0,\ 0.4^2)$, and assume that  $\hat{\mu}$ is known and  $\hat{\sigma}^2$ is unknown.
    \item[Case2] The true parameter is $(\hat{\mu},\ \hat{\sigma}^2) = (0.4,\ 0.4^2)$, and assume that  $\hat{\mu}$ is unknown and  $\hat{\sigma}^2$ is known.
\end{description}
In Case1, we used $\pi(\hat{\sigma}^{-2}) = Gamma(3, 0.48)$ as the prior distribution of $\hat{\sigma}^{-2}$.
Similarly, in Case2, we used $\pi(\hat{\mu}) = \mathcal{N}(0, 0.8^2)$ as the prior distribution of $\hat{\mu}$.
Note that   posterior distributions of $g_t(\bm{s} \mid \bm{\theta}_{\bm{x}})$ in Case1 and Case2 are given by $t$-distribution and normal distribution, respectively (see, e.g., \cite{bishop2006pattern}).

The experiment results are shown in Figure \ref{fig:unknown_noise}.
%
From Figure \ref{fig:unknown_noise}, even in this setting, we can  confirm  that the proposed method has better performance than  existing methods.
\begin{figure}[t]
    \begin{center}
        \includegraphics[scale=0.26]{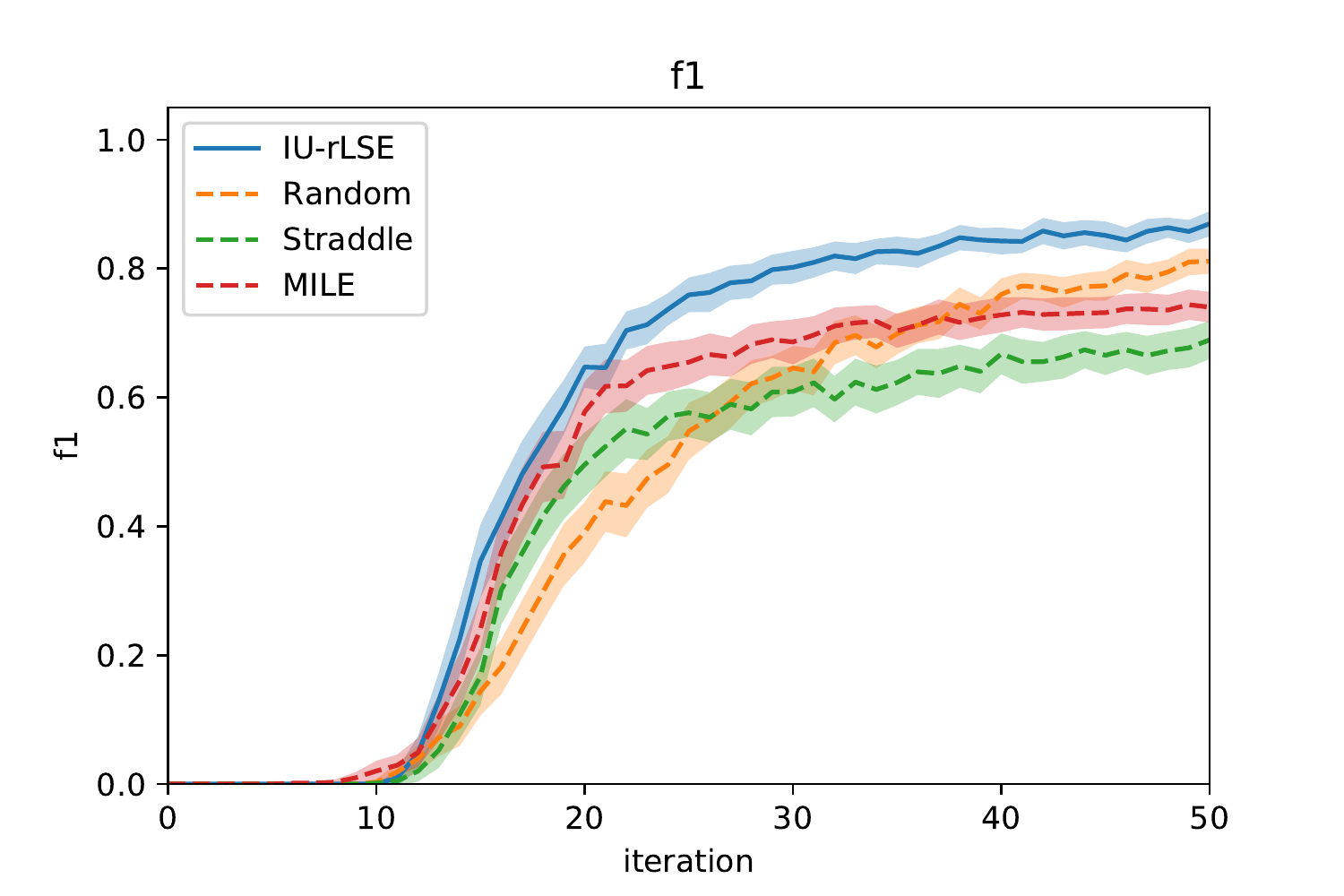}
        \includegraphics[scale=0.26]{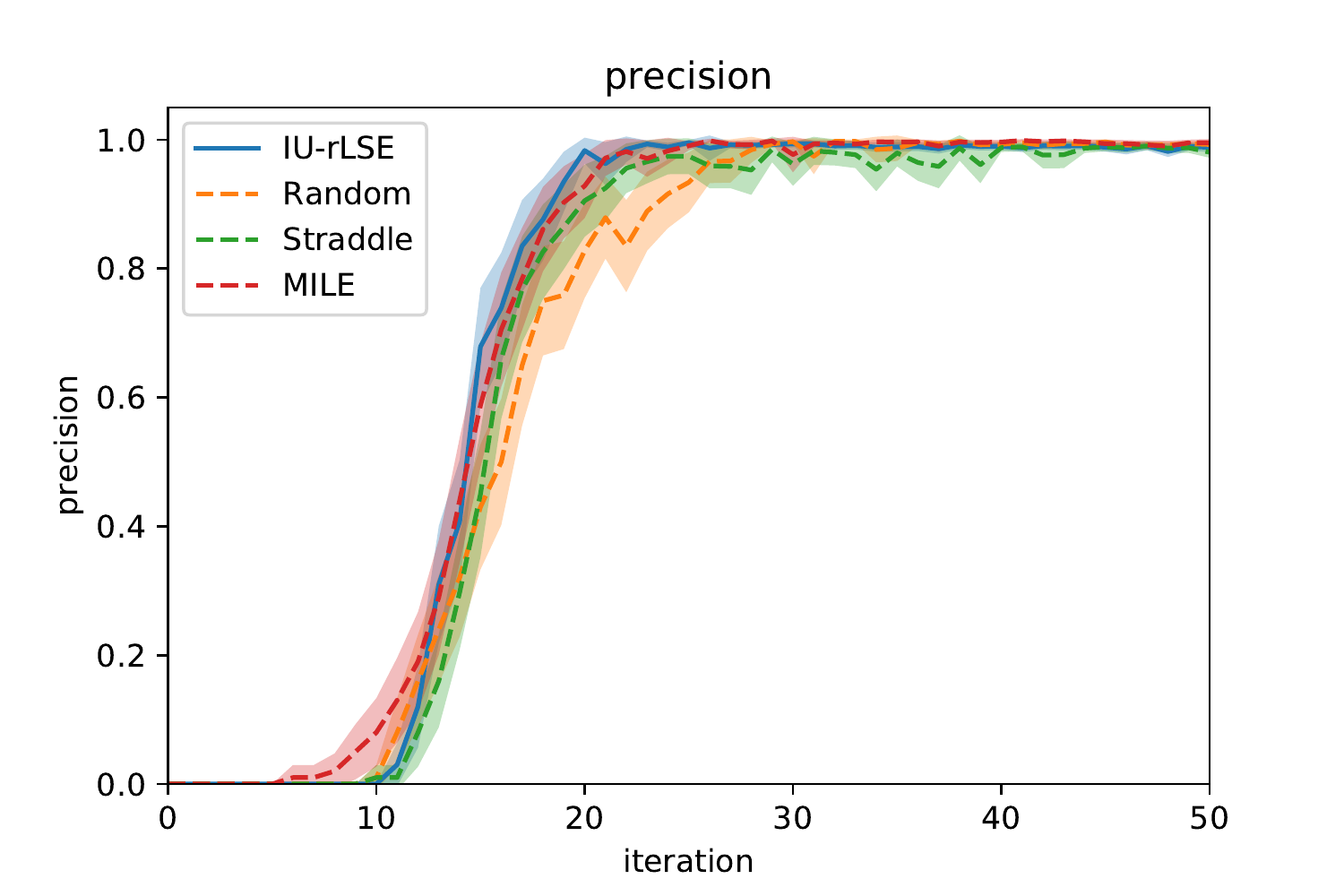}
        \includegraphics[scale=0.26]{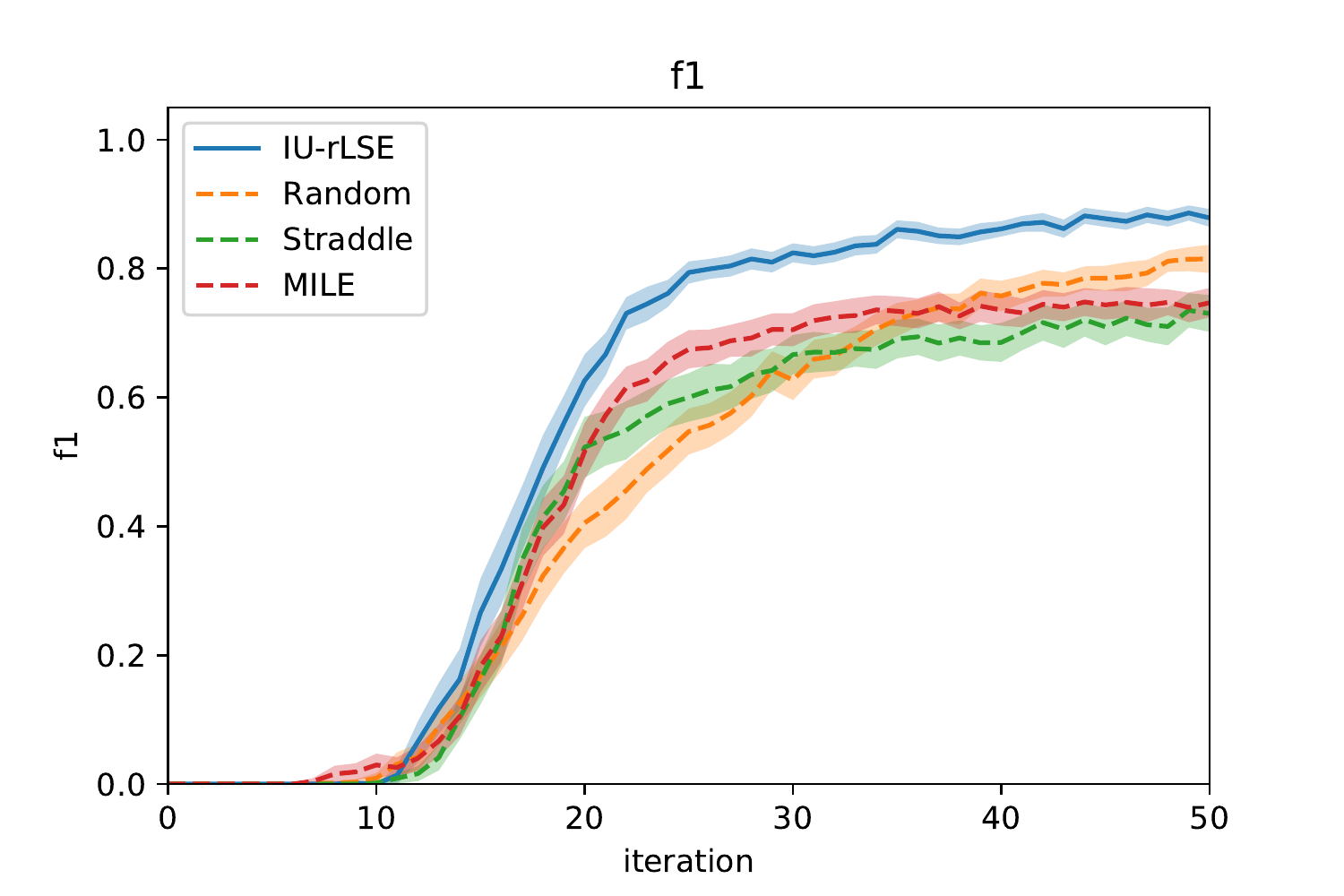}
        \includegraphics[scale=0.26]{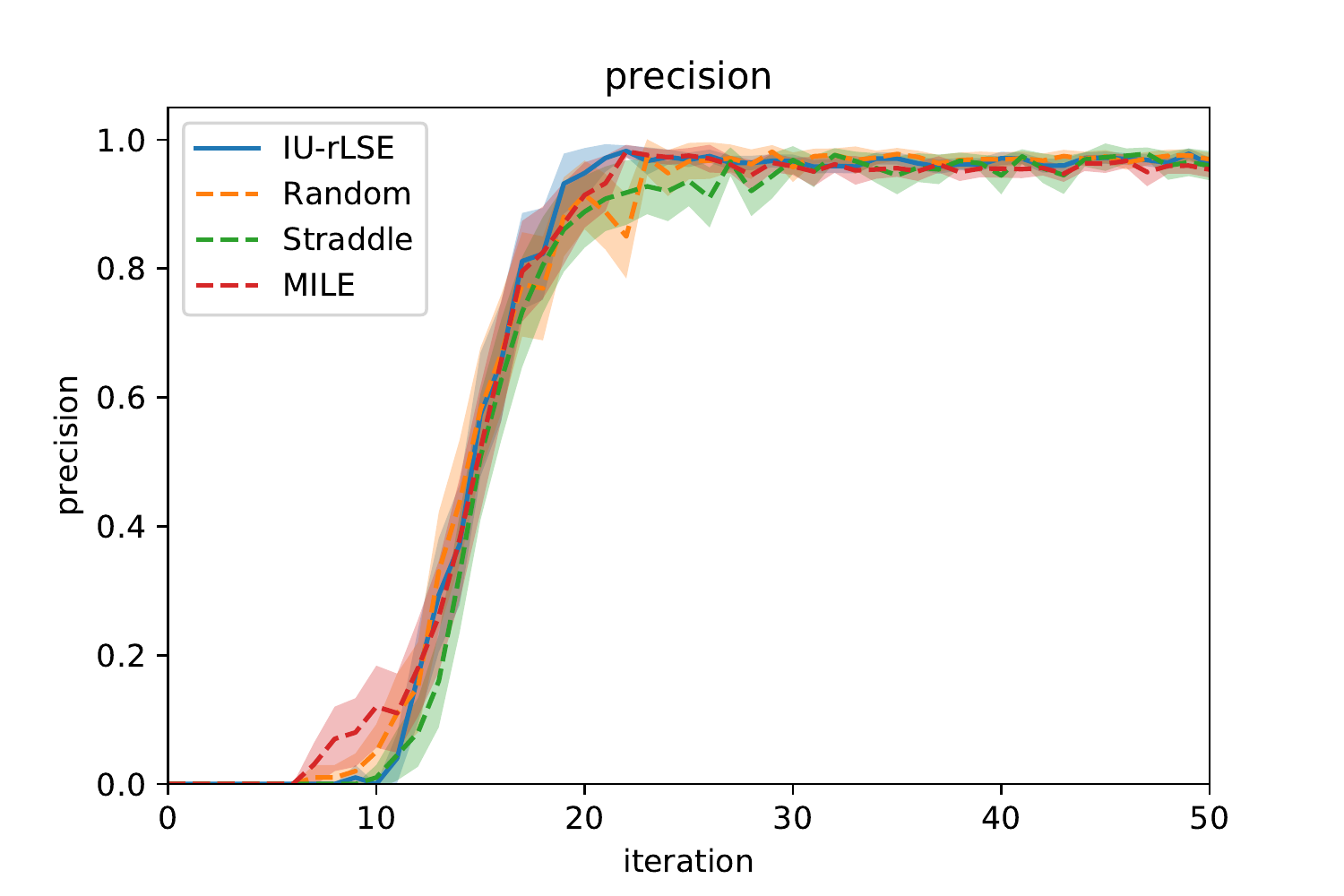}
        \caption{Experiment results based on 100 Monte Carlo simulations.
First and second (third and fourth) figures represent the results in Case1 (Case2).
The $F1$-score in each case is shown in first and third figures, and precision in each case is shown in second and fourth figures.}
        \label{fig:unknown_noise}
    \end{center}
\end{figure}

\subsection{Real-Data Experiment}
In this subsection, we confirmed the classification accuracy by using the Combined Cycle Power Plant (CCPP) dataset \cite{Dua:2019,article,inproceedings}.
CCPP contains $ 9568 $ instances and consists of four parameters (Temperature, Ambient Pressure, Relative humidity, Exhaust Vaccume) representing the state in CCPP as inputs, and the amount of power generation with respect to time average as the output.
Here, accurate control of CCPP state parameters is difficult due to environmental factors and control errors, and there is  input uncertainty.
We first standardized the output of each instance to  average  0, and normalized each input feature to  average  0  and variance 1.
In this experiment, we first extracted $ 7568 $ data randomly, calculated the posterior mean of GP using this, and
 considered  it as the true function.
The remaining 2000 data were used as the set of candidate points   $ \mathcal {X} $.
We used Gaussian kernel with
  $\sigma_f^2 = 300$ and $ L = 2$, and set
 $\sigma^2 = 0.5$ and $h = -15$.
As the input distribution, we used
 $\bm{S}(\bm{x}) = \bm{x} + \mathcal{N}(0,\ 0.125^2)$.
The experiment results based on 20 Monte Carlo simulations are shown in Figure \ref{fig:power_plant}.
From Figure \ref{fig:power_plant} on left, we can see that the $F1$-score for  the proposed method is larger than those of existing methods.
 Furthermore, we performed the similar experiment as in Subsection \ref{subsub:1dimf}.
From Figure \ref{fig:power_plant} on right, we can see that  precision of the proposed method tends to 1.
On the other hand, we can also see that precision of existing methods (with focus on the classification of $f$)  do not tend to 1.
\begin{figure}
    \begin{center}
        \includegraphics[scale=0.24]{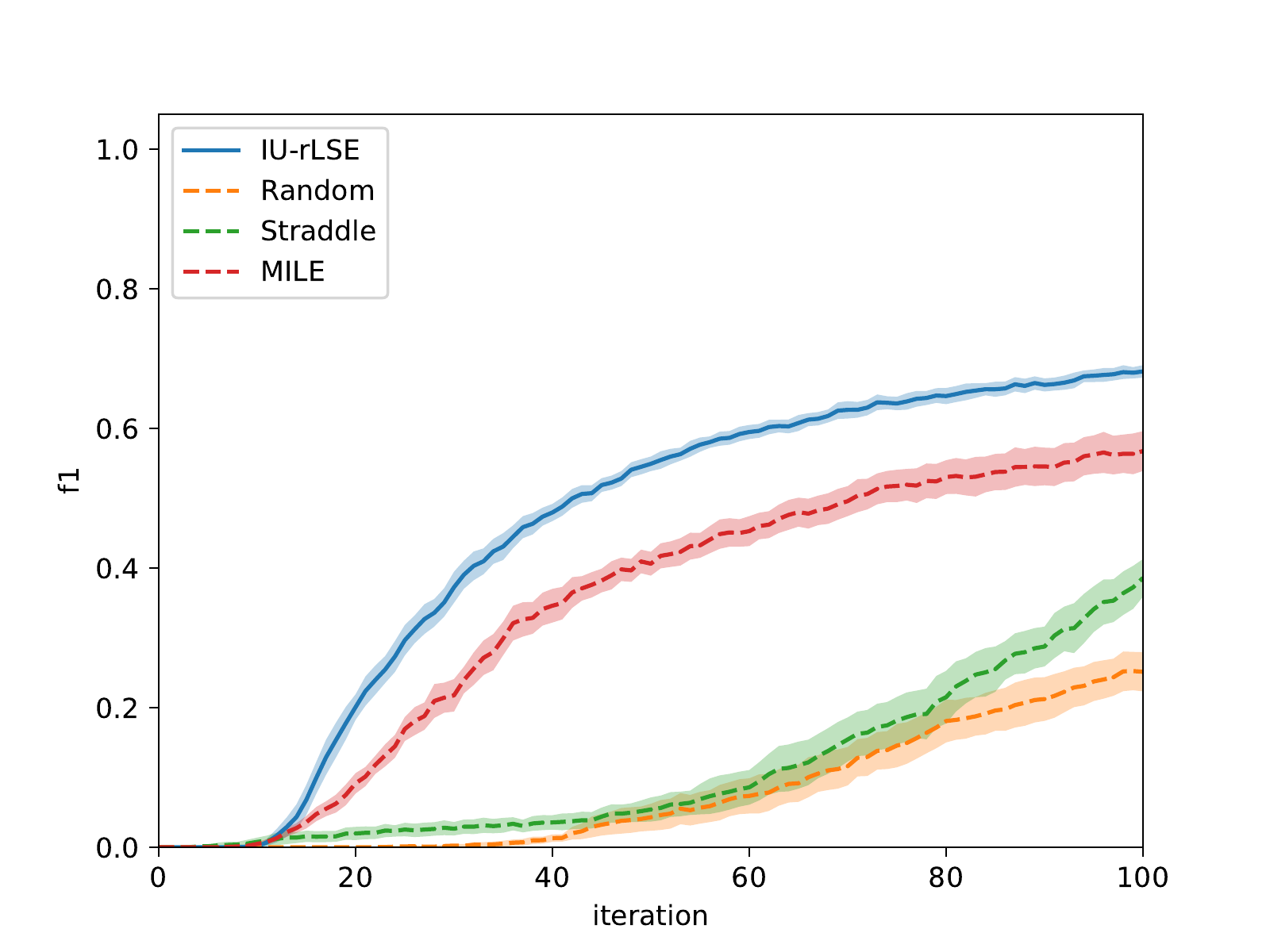}
        \includegraphics[scale=0.24]{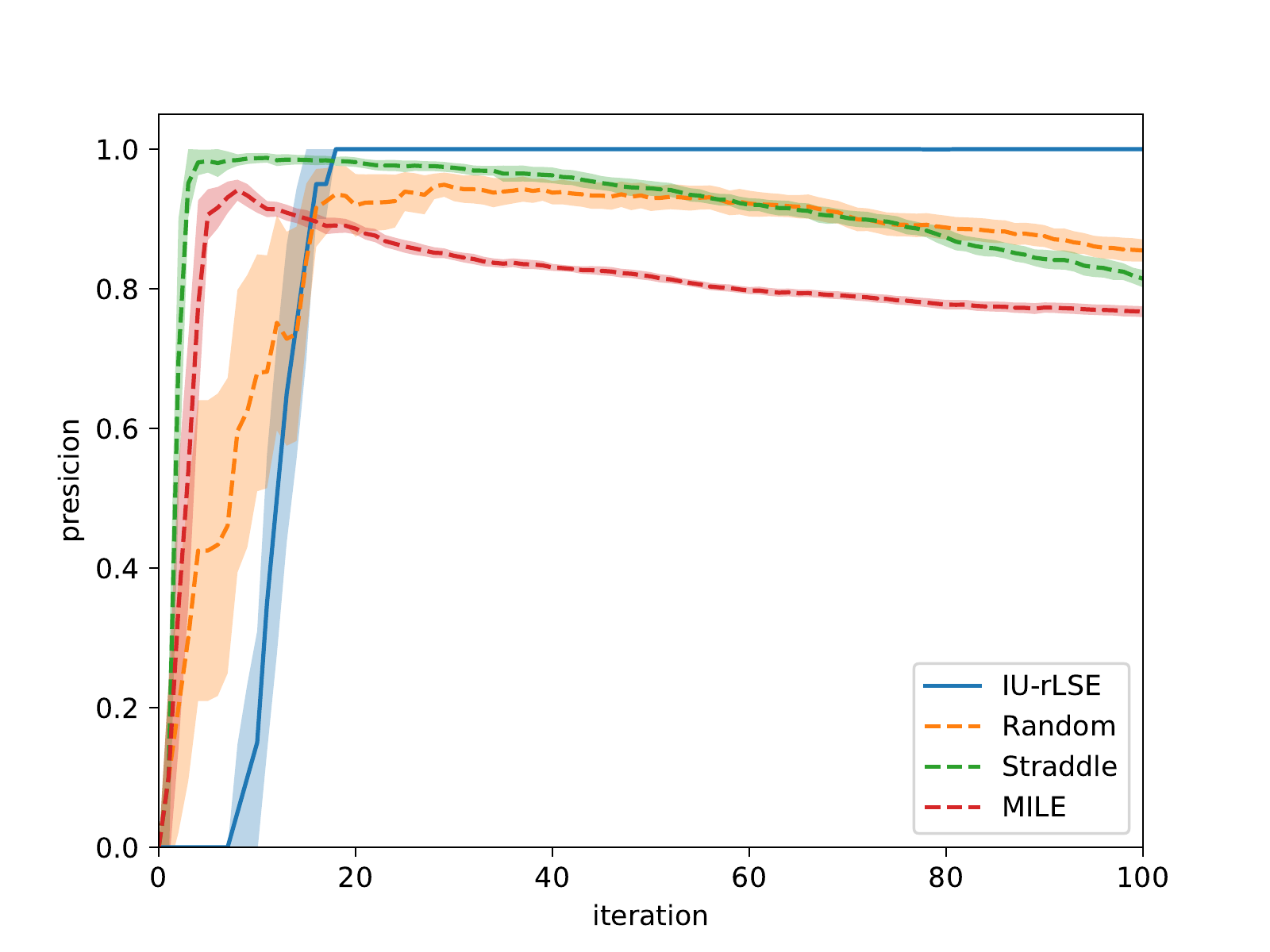}
        \caption{
Average $F1$-score(left) and precision(right) for the Combined Cycle Power Plant Dataset based on $20$ Monte Carlo simulations.}
        \label{fig:power_plant}
    \end{center}
\end{figure}

%% file: section6.tex
\section{Conclusion}
We considered the problem for identifying input points where probabilities that the black-box function $f$ falls below the threshold $h$ are more than $ \alpha $ in the situation which inputs have uncertain.
We proposed the level set estimation method and acquisition functions by assuming GP  as the prior distribution of $ f $ and constructing  credible intervals for probabilities that $ f $ falls below the threshold $ h $ under input uncertainty.
Through theoretical analysis and numerical experiments, it was confirmed that the proposed method has better performance than other methods.

%% file: acknowledements.tex
\section*{Acknowledgements}
This work was partially supported by MEXT KAKENHI (17H00758, 16H06538), JST CREST (JPMJCR1302, JPMJCR1502), RIKEN Center for Advanced Intelligence Project, and JST support program for starting up innovation-hub on materials research by information integration initiative.

%% file: bib.tex
\bibliographystyle{apalike}
\bibliography{myref}

%% file: appendix.tex
\appendix
\section{Derivation of Proposed Method}\label{sec:tech_proof}
\subsection{Deteils of Estimation about $\mathcal{H}$}
In this subsection, we discuss the details of estimating $\mathcal{H}$ based on $p_{t,\bm{x}}$.
First, we prove the existence of $p_{t,\bm{x}}$.
\begin{lemma}\label{lem:int_process}
   There exists a random variable $p_{t,\bm{x}}$.
\end{lemma}
\begin{proof}
From the definition of $p_{t,\bm{x}}$, it is sufficient to show that the  integral
    \begin{align*}
        \int_{\mathcal{D} \times \mathcal{D}} |\Cov(\1[f_t(\bm{s}) <&\ h], \1[f_t(\bm{s}^{\prime}) < h])| \nonumber \\
                &g(\bm{s} \mid \bm{\theta}_{\bm{x}})g(\bm{s}^{\prime} \mid \bm{\theta}_{\bm{x}})d\bm{s}d\bm{s}'
    \end{align*}
is finite (\cite{papoulis2002probability}, Chapter10).
    Noting that  $\1[f_t(\bm{s}) < h]  \in \{0, 1\}$, we have $|\Cov(\1[f_t(\bm{s}) < h], \1[f_t(\bm{s}^{\prime}) < h])| \leq 1.$
Therefore, we get
    \begin{alignat*}{2}
        & & \int_{\mathcal{D} \times \mathcal{D}} |\Cov(\1[f_t(\bm{s}) <&\ h], \1[f_t(\bm{s}^{\prime}) < h])|  \\
        & & &g(\bm{s} \mid \bm{\theta}_{\bm{x}})g(\bm{s}^{\prime} \mid \bm{\theta}_{\bm{x}})d\bm{s}d\bm{s}' \\
        &\leq& \int_{\mathcal{D} \times \mathcal{D}} g(\bm{s} \mid \bm{\theta}_{\bm{x}})g(\bm{s}^{\prime} \mid &\bm{\theta}_{\bm{x}})d\bm{s}d\bm{s}' \\
         &=& 1\hspace{80pt}& \\
         &<& +\infty.\hspace{66pt} &
    \end{alignat*}
\end{proof}
Next, the following lemma holds:
\begin{lemma}\label{lem:chebyshev}
    Let $\delta \in (0, 1)$.
    Then, with probability at least $1 - \delta$, it holds that
    \begin{equation*}
        |p_{t, \bm{x}} - \mu_t^{(p)}(\bm{x})| < \delta^{-\frac{1}{2}} \gamma_t(\bm{x})
    \end{equation*}
   where $ \mu_t^{(p)}(\bm{x})$ and $\gamma^2_t(\bm{x})$ are given by \eqref{eq:def-mutp} and \eqref{eq:def-gammat2}, respectively.
\end{lemma}
\begin{proof}
    From Chebyshev's inequality,  for any $\epsilon > 0$, it holds that
    \begin{equation}
        {\rm Prob}[|p_{t, \bm{x}} - \mu_t^{(p)}(\bm{x})| \geq \epsilon] \leq \frac{     {\rm{Var}} [ p_{t,{\bm{x}} }]     }{\epsilon^2}.
        \label{eq:chebyshev}
    \end{equation}
Moreover, noting that $\Cov [X,Y] \leq (\V [X] + \V[Y])/2$, we obtain
    \begin{align}
        &\quad {\rm{Var}} [ p_{t,{\bm{x}} }] & \nonumber \\
        &= \int \int_{\D \times \D} \Cov(\1[f_t(\bm{s}) < h], \1[f_t(\bm{s}^{\prime}) < h]) \nonumber \\
            &\hspace{120pt} g(\bm{s}|\bm{\theta}_{\bm{x}})g(\bm{s}^{\prime}|\bm{\theta}_{\bm{x}})d\bm{s}d\bm{s}^{\prime} \nonumber \\
        &\leq \int \int_{\D \times \D} \frac{\V[\1[f_t(\bm{s}) < h]] + \V[\1[f_t(\bm{s}^{\prime}) < h]]}{2} &\nonumber \\
            & \hspace{120pt} g(\bm{s}|\bm{\theta}_{\bm{x}})g(\bm{s}^{\prime}|\bm{\theta}_{\bm{x}})d\bm{s}d\bm{s}^{\prime} \nonumber \\
        &=\int_{\D} \V[\1[f_t(\bm{s}) < h]]g(\bm{s}|\bm{\theta}_{\bm{x}})d\bm{s}.&
        \label{eq:ub_cov}
    \end{align}
    Hence, by combining  (\ref{eq:chebyshev}) and  (\ref{eq:ub_cov}) we get
    \begin{align*}
        &\quad {\rm Prob}[|p_{t, \bm{x}} - \mu_t^{(p)}(\bm{x})| \geq \epsilon] \\
        &\leq \frac{\int_{\D} \V[\1[f_t(\bm{s}) < h]]g(\bm{s}|\bm{\theta}_{\bm{x}})d\bm{s}}{\epsilon^2}  \\
        &= \frac{  \gamma^2_t(\bm{x})  }{\epsilon ^2}.
    \end{align*}
    Therefore, putting $\epsilon = \delta^{\frac{1}{2}} \gamma_t(\bm{x})$, the following holds  with probability at least $1-\delta$:
    \begin{equation*}
        |p_{t, \bm{x}} - \mu_t^{(p)}(\bm{x})| < \delta^{-\frac{1}{2}} \gamma_t(\bm{x}).
    \end{equation*}
\end{proof}

\subsection{Details of Aquisition Function}\label{subsec:aq_derive}
In this subsection, we derive several lemmas on the acquisition function.
First, the following lemma holds:
\begin{lemma}\label{lem:second_inequity}
Let $0 < \alpha < 1$ and  $\epsilon > 0$ with   $0 < \alpha - \epsilon < 1$.
Also let $\Phi_{\overline{\bm{s}}} \in [0, 1]$.
Then, the solution of the inequality
\begin{equation*}
\Phi_{\overline{\bm{s}}} - \beta^{\frac{1}{2}}\sqrt{\Phi_{\overline{\bm{s}}}(1 - \Phi_{\overline{\bm{s}}})} > \alpha - \epsilon
\end{equation*}
 is given by
\begin{equation*}
    c < \Phi_{\overline{\bm{s}}} \leq 1.
\end{equation*}
where
\begin{equation*}
    c = \frac{2(\alpha - \epsilon) + \beta + \sqrt{\beta^2 + 4(\alpha - \epsilon)\beta - 4(\alpha- \epsilon)^2\beta}}{2(1 + \beta)}
\end{equation*}
\end{lemma}
\begin{proof}
First, the inequality
    \begin{equation*}
\Phi_{\overline{\bm{s}}} - (a - \epsilon)  > \sqrt{\beta \Phi_{\overline{\bm{s}}}(1 - \Phi_{\overline{\bm{s}}})}  \label{eq:ineq-1}
    \end{equation*}
holds because $\Phi_{\overline{\bm{s}}} - \beta^{\frac{1}{2}}\sqrt{\Phi_{\overline{\bm{s}}}(1 - \Phi_{\overline{\bm{s}}})} > \alpha - \epsilon$. Furthermore, since
     $\sqrt{\beta\Phi_{\overline{\bm{s}}}(1 - \Phi_{\overline{\bm{s}}})} > 0$, it holds that
    \begin{equation}
        \Phi_{\overline{\bm{s}}} > (a - \epsilon) .
        \label{eq:alpha_cond}
    \end{equation}
On the other hand, \eqref{eq:ineq-1} can be rewritten as
    \begin{alignat}{2}
        &\quad& \Phi_{\overline{\bm{s}}} - (a - \epsilon)  > \sqrt{\beta \Phi_{\overline{\bm{s}}}(1 - \Phi_{\overline{\bm{s}}})} &\nonumber \\
        &\Leftrightarrow& (1 + \beta)\Phi_{\overline{\bm{s}}}^2 - \{2(a - \epsilon) + \beta\}\Phi_{\overline{\bm{s}}}&  \nonumber \\
        &        &+ (\alpha - &\epsilon)^2 > 0  \label{eq:ineq-rewrite}.
    \end{alignat}
   Thus, by  using the quadratic formula, the solution of \eqref{eq:ineq-rewrite} is given by
        \begin{equation}
            \Phi_{\overline{\bm{s}}} < \Phi_{\overline{\bm{s}}}^-,\ \Phi_{\overline{\bm{s}}}^+ < \Phi_{\overline{\bm{s}}},
            \label{eq:ineq_ans}
        \end{equation}
        where
        \begin{align*}
            \Phi_{\overline{\bm{s}}}^-&=\frac{2(\alpha - \epsilon) + \beta - \sqrt{\beta^2 + 4(\alpha - \epsilon)\beta - 4(\alpha- \epsilon)^2\beta}}{2(1 + \beta)}, \\
            \Phi_{\overline{\bm{s}}}^+&=\frac{2(\alpha - \epsilon) + \beta + \sqrt{\beta^2 + 4(\alpha - \epsilon)\beta - 4(\alpha- \epsilon)^2\beta}}{2(1 + \beta)}.
        \end{align*}
        Moreover,  $\Phi_{\overline{\bm{s}}}^-$ and $ \Phi_{\overline{\bm{s}}}^+$ satisfy
        \begin{align*}
            \Phi_{\overline{\bm{s}}}^-&=\frac{2(\alpha - \epsilon) + \beta - \sqrt{\beta^2 + 4(\alpha - \epsilon)\beta - 4(\alpha- \epsilon)^2\beta}}{2(1 + \beta)} \\
                                       &\leq\frac{2(\alpha - \epsilon) + \beta - \sqrt{\beta^2}}{2(1 + \beta)} \\
                                       &=\frac{2(\alpha - \epsilon)}{2(1 + \beta)} \\
                                       &\leq\frac{2(\alpha - \epsilon)(1 + \beta)}{2(1 + \beta)} \\
                                       &=\alpha - \epsilon
        \end{align*}
and
        \begin{align*}
            \Phi_{\overline{\bm{s}}}^+ &= \frac{2(\alpha - \epsilon) + \beta + \sqrt{\beta^2 + 4(\alpha - \epsilon)\beta - 4(\alpha- \epsilon)^2\beta}}{2(1 + \beta)} \\
                                       &\geq \frac{2(\alpha - \epsilon) + \beta + \sqrt{\beta^2}}{2(1 + \beta)} \\
                                       &= \frac{2(\alpha - \epsilon) + 2\beta}{2(1 + \beta)} \\
                                       &\geq \frac{2(\alpha - \epsilon) + 2(\alpha - \epsilon)\beta}{2(1 + \beta)} \\
                                       &= \alpha - \epsilon.
        \end{align*}
Next, we assume $\Phi_{\overline{\bm{s}}}^+ > 1$.
Then, \eqref{eq:ineq-rewrite}  does not have any solutions on
$[\Phi_{\overline{\bm{s}}}^-, 1]$.
        However, \eqref{eq:ineq-rewrite}  holds  when  $\Phi_{\overline{\bm{s}}} = 1$.
        This is a contradiction.
Hence, we get $\Phi_{\overline{\bm{s}}}^+ \leq 1$.
       This implies that
        \begin{equation}
            \Phi_{\overline{\bm{s}}}^- < \alpha - \epsilon,\ \alpha - \epsilon \leq \Phi_{\overline{\bm{s}}}^+ \leq 1.
            \label{eq:pmphi_cond}
        \end{equation}
        Finally, from (\ref{eq:alpha_cond}),  (\ref{eq:ineq_ans}) and  (\ref{eq:pmphi_cond}) we obtain
        \begin{equation*}
  \Phi_{\overline{\bm{s}}}^+ < \Phi_{\overline{\bm{s}}} \leq 1,
  \end{equation*}
  \begin{equation*}
\Phi_{\overline{\bm{s}}}^+ = \frac{2(\alpha - \epsilon) + \beta + \sqrt{\beta^2 + 4(\alpha - \epsilon)\beta - 4(\alpha- \epsilon)^2\beta}}{2(1 + \beta)}
    \end{equation*}
\end{proof}

Next, we derive a lemma on the exact form of the integral in the acquisition function.
\begin{lemma}\label{lem:drive_aq}
    Let $p(y^* \mid \bm{s}^*)$ be a probability density function of normal distribution with mean  $\mu_t(\bm{s}^*)$ and variance $\sigma_t^2(\bm{s}^*) + \sigma^2$.
Then, it holds that
    \begin{align}
        &\sum_{\overline{\bm{s}} \in \overline{\mathcal{S}}}\int \1[\mu_t{(\overline{\bm{s}}\mid\bm{s}^*, y^*)} > h - \sigma_t(\overline{\bm{s}}\mid\bm{s}^*)\Phi^{-1}(c)] \nonumber \\
        \label{eq:org_int}
        &\hspace{145pt}p(y^* \mid \bm{s}^*) dy^* \\
        =& \sum_{\overline{\bm{s}} \in \overline{\mathcal{S}}} \Phi\Biggl(\frac{\sqrt{\sigma_{t}^{2}(\bm{s}^{*})+\sigma^{2}}}{\left|k_{t}\left(\overline{\bm{s}},
        \boldsymbol{s}^{*}\right)\right|} (\mu_{t}(\overline{\bm{s}}) \nonumber \\
        & \hspace{105pt} - \Phi^{-1}(c) \sigma_{t}\left(\overline{\bm{s}} | \boldsymbol{s}^{*}\right)-h)\Biggr) \nonumber
    \end{align}
   where
    \begin{equation}
        \label{eq:post_mu}
        \mu_t{(\overline{\bm{s}}\mid\bm{s}^*, y^*)} = \mu_t(\overline{\bm{s}}) - \frac{k_t(\overline{\bm{s}},\ \bm{s}^*)}{\sigma_t^2(\bm{s}^*) + \sigma^2}(y^* - \mu_t(\bm{s}^*)).
    \end{equation}
\end{lemma}
\begin{proof}
By substituting (\ref{eq:post_mu}) into the indicator function in (\ref{eq:org_int}),
we have
    \begin{align*}
        & \mu_t{(\overline{\bm{s}}\mid\bm{s}^*, y^*)} > h - \sigma_t(\overline{\bm{s}}\mid\bm{s}^*)\Phi^{-1}(c) \\
        \Leftrightarrow& \mu_t(\overline{\bm{s}}) - \frac{k_t(\overline{\bm{s}},\ \bm{s}^*)}{\sigma_t^2(\bm{s}^*) + \sigma^2}(y^* - \mu_t(\bm{s}^*)) \\
                    & \hspace{80pt} > h - \sigma_t(\overline{\bm{s}}\mid\bm{s}^*)\Phi^{-1}(c) \\
        \Leftrightarrow& \frac{k_t(\overline{\bm{s}},\ \bm{s}^*)}{\sigma_t^2(\bm{s}^*) + \sigma^2}(y^* - \mu_t(\bm{s}^*)) \\
                    & \hspace{80pt} <  \mu_t(\overline{\bm{s}}) + \sigma_t(\overline{\bm{s}}\mid\bm{s}^*)\Phi^{-1}(c) -h.
    \end{align*}
    Next, let
    \begin{align*}
        y^L(\overline{\bm{s}}) &= \frac{\sigma_t^2(\bm{s}^*) + \sigma^2}{k_t(\overline{\bm{s}},\ \bm{s}^*)} \left( \mu_t(\overline{\bm{s}}) + \sigma_t(\overline{\bm{s}}\mid\bm{s}^*)\Phi^{-1}(c) -h\right) \\
        & \hspace{160pt} + \mu_t(\bm{s}^*).
    \end{align*}
    Then, (\ref{eq:org_int}) can be written as
    \begin{align*}
    &\sum_{\bm{x} \in \mathcal{X}}\int \1[\mu_t{(\overline{\bm{s}}\mid\bm{s}^*, y^*)} > h - \sigma_t(\overline{\bm{s}}\mid\bm{s}^*)\Phi^{-1}(c)] \\
    &\hspace{145pt}p(y^* \mid \bm{s}^*) dy^* \\
    =& \sum_{\overline{\bm{s}} \in \overline{\mathcal{S}},\ k_t(\overline{\bm{s}},\ \bm{s}^*) \geq 0} \int_{y^L(\overline{\bm{s}})}^{+\infty} p(y^* \mid \bm{s}^*) dy^* + \\
    & \hspace{50pt} \sum_{\overline{\bm{s}} \in \overline{\mathcal{S}},\ k_t(\overline{\bm{s}},\ \bm{s}^*) < 0} \int_{-\infty}^{y^L(\overline{\bm{s}})} p(y^* \mid \bm{s}^*) dy^*.
    \end{align*}
    Therefore, noting that $p(y^* \mid \bm{s}^*)$ is the normal density function, from symmetry of normal distribution we get
    \begin{align*}
        & \sum_{\overline{\bm{s}} \in \overline{\mathcal{S}}} \Phi\Biggl(\frac{\sqrt{\sigma_{t}^{2}(\bm{s}^{*})+\sigma^{2}}}{\left|k_{t}\left(\overline{\bm{s}},
        \boldsymbol{s}^{*}\right)\right|} (\mu_{t}(\overline{\bm{s}}) \nonumber \\
        & \hspace{105pt} - \Phi^{-1}(c) \sigma_{t}\left(\overline{\bm{s}} | \boldsymbol{s}^{*}\right)-h)\Biggr). \nonumber
    \end{align*}
\end{proof}

%% file: sec2-utf8.tex
\section{Proof of Theorem\ref{thm:seido}}\label{sec:appendix_seido}
\begin{proof}
From Lemma \ref{lem:chebyshev}, putting $\beta^{1/2} = (\delta/|\mathcal{X} | ) ^{-1/2}$, for any
   ${\bm{x}} \in \mathcal{X} $ it holds that
$ p^\ast _{\bm{x}}  \in Q_T ({\bm{x}} ) $ with probability at least
$1- \delta /|\mathcal{X} | $, where
 $T$ means $t$ at the end of the algorithm.
Hence, with probability at least $1-\delta $, for any ${\bm{x}} \in \mathcal{X}$ it holds that
$ p^\ast _{\bm{x}}  \in Q_T ({\bm{x}} ) $.
Therefore, by combining
 this result, the classification rule and the definition of
$e_\alpha ({\bm{x}} )$, we get Theorem \ref{thm:seido}.
\end{proof}

%% file: sec3-utf8.tex
\section{Proof of Theorem\ref{thm:owaru}}\label{sec:appendix_owaru}
In this section, we derive a theorem on convergence properties of the algorithm.
First, we define several notations.
For each ${\bm{x}} \in \mathcal{X} $, let
	\begin{align*}
		D_{\bm{x}} =    \{  {\bm{x}}' \in D \ | \ \f \xi >0, \  \PR( {\bm{S}} ({\bm{x}} ) \in \mathscr {N} ( {\bm{x}}' ; \xi ) ) >0  \},
	\end{align*}
where $ \mathscr {N} ( {\bm{x}}' ; \xi ) \equiv \{ {\bm{a}} \in D \ | \ \| {\bm{a}} - {\bm{x}}' \|  <\xi \}$.
Thus, $D_{\bm{x}}$ is the set of points that can be observed when  ${\bm{x}} $ is observed.
Then, define $\tilde{D} $ as follows:
\begin{align*}
\tilde{D} =    \bigcup _{ {\bm{x}} \in \mathcal{X} }    D_{\bm{x}} .
\end{align*}
Furthermore, let ${\bm{A}} _t $ be an input random variable at  $t^{\rm th}$ trial, and let
  $Y_{{\bm{A}}_t }$ be an output random variable corresponding to ${\bm{A}} _t $.
Then, define $ \hat{\mu }_t $ as a posterior mean function based on the data  $\{ ({\bm{A}}_i , Y_{{\bm{A}}_i } ) \} ^t_{i=1}$.

Next, we assume the following four conditions:
\begin{description}
\item [(A1)] Probabilities $\{ \eta_t  \} _ {t \in \N} $ satisfy     $\sum_{t=1}^\infty \eta_t = \infty$.
\item [(A2)] For any $\xi >0$, there exists  $ \delta_ \xi >0$ such that
	\begin{align}
 		\limsup_{t \to \infty } \max_ {  {\bm{x}} \in \mathcal{X} }  \int   _{  \hat{\mu}^{-1}_t (    (h -\delta _\xi , h+\delta_\xi ) )   }     g({\bm{s}} | {\bm\theta}_{\bm{x}} )  d {\bm{s}}         < \xi,  \label{eq:as1mean}
	\end{align}
with probability 1.
\item [(A3)]  For any ${\bm{x}} \in \tilde{D}$, the kernel function $k$ is continuous at $(\bx,\bx )$.
\item [(A4)]  For any $\xi >0$ and ${\bm{x}} \in \tilde{D}$, there exists   $\delta _{\xi, {\bm{x}} } >0$ such that
$ |\sigma^2_t ( {\bm{x}} ) - \sigma^2 _t ({\bm{x}}')| < \xi $ for any
$t \geq 1$, ${\bm{x}}_1,\ldots , {\bm{x}}_t \in \tilde{D} $ and ${\bm{x}}'  \in \mathscr{N} ( {\bm{x}} ; \delta _{\xi, {\bm{x}} })$.
\end{description}

Condition ${\sf (A1)}$ satisfies when each $\eta_t $ is greater than a positive constant $c$.
Similarly, when $ \eta_t = o( t^{-1} )$,  ${\sf (A1)}$ also holds.
Condition  ${\sf (A2)}$ requires that
the probability that an input point falls in a region where the posterior mean approaches the threshold $ h $ can be  reduced  sufficiently when   $ \delta_ \xi $ becomes small.
Condition ${\sf (A3)}$ requires that the kernel function $k$ is continuous on  $\tilde{D} \times \tilde{D}$ and
 ${\sf (A4)}$ requires the equicontinuity for the sequence of posterior variances.
Under these conditions, Theorem\ref{thm:owaru} holds.
The proof is given in  Subsection \ref{sub1}--\ref{sub3}.

%% file: subsec3-1-utf8.tex
\subsection{Preparation of the proof}\label{sub1}
In this subsection, we provide two lemmas for proving Theorem \ref{thm:owaru}.
First, for any   finite subset  $\Omega$ of  $\tilde{D}$, the following lemma holds:
	\begin{lemma}\label{lem:as1cov}
	Assume that  conditions {\sf (A1)} -- {\sf (A4) } hold. Then, with probability 1, for any ${\bm{x}} \in \Omega$, it holds that
	$$
	\sigma^2_t ({\bm{x}}) \to 0 \q (\text{as} \ t \to \infty).
	$$
	\end{lemma}
The proof is same as that of Theorem 4.2 in \cite{inatsu2019active}, we omit the details.

Next, the following lemma on the compactness of $\tilde{D}$ holds:
	\begin{lemma}\label{lem:compact}
	The set $\tilde{D}$ is compact.
	\end{lemma}
\begin{proof}
From the definition of $\tilde{D}$, the set $\tilde{D} $ satisfies  $\tilde{D}   \subset D$.
 In addition, noting that $D$ is bounded,  we have that $\tilde{D}$ is also bounded.
Hence, it is sufficient to show that $\tilde{D}$ is a closed set.
Let $\text{cl} (\tilde{D})$ be a closure of $\tilde{D}$.
Then, we prove $\tilde{D} = \text{cl} (\tilde{D})$.
From the definition of the closure, we get  $\tilde{D} \subset \text{cl} (\tilde{D})$.
Next, we show
$ \text{cl} (\tilde{D}) \subset \tilde{D} $.
Let ${\bm{x}} $ be an arbitrary point of $ \text{cl} (\tilde{D})$.
then, since the number of elements in $\mathcal{X}$ is finite,
the following formula holds:
$$
\text{cl} (\tilde{D} ) = \text{cl} \left (  \bigcup _{{\bm{a}} \in \mathcal{X} } D_{\bm{a}}  \right ) = \bigcup_{ {\bm{a}} \in \mathcal{X} } \text{cl} ({D}_{\bm{a}} ).
$$
Thus, there exists ${\bm{x}}' \in \mathcal{X}$ such that ${\bm{x}} \in \text{cl} ({D} _{ {\bm{x}}' } )$.
Therefore, for any $\xi >0$, it holds that
   $\mathscr{N} ({\bm{x}} ;\xi )   \cap {D} _{ {\bm{x}}' }  \neq \emptyset $.
Hence, there exists ${\bm{x}}'' \in {D}_{ {\bm{x}}' } $ such that ${\bm{x}}'' \in \mathscr{N} ({\bm{x}} ;\xi ) $.
Moreover, noting that  $ \mathscr{N} ({\bm{x}} ;\xi )$ is an open set, there exists   $\eta >0$ such that
$ \mathscr{N} ( {\bm{x}}'' ; \eta )   \subset \mathscr{N} ({\bm{x}} ;\xi ) $.
On the other hand, since  ${\bm{x}}'' $ is an element of $D_{{\bm{x}}'} $, we have
$\PR (  {\bm{S}} ({\bm{x}}') \in \mathscr{N} ( {\bm{x}}'' ; \eta ) ) >0.$
Recall that
  $\mathscr{N} ( {\bm{x}}'' ; \eta ) $ satisfies
$ \mathscr{N} ( {\bm{x}}'' ; \eta )   \subset \mathscr{N} ({\bm{x}} ;\xi ) $.
Therefore, by using
$$
  {\bm{S}} ({\bm{x}}') \in \mathscr{N} ( {\bm{x}}'' ; \eta )  \Rightarrow
  {\bm{S}} ({\bm{x}}') \in \mathscr{N} ({\bm{x}} ;\xi ),
$$
we obtain
$$
\PR(   {\bm{S}} ({\bm{x}}') \in \mathscr{N} ({\bm{x}} ;\xi ) ) \geq \PR (  {\bm{S}} ({\bm{x}}') \in \mathscr{N} ( {\bm{x}}'' ; \eta ) ) >0.
$$
Hence, we get ${\bm{x}} \in D_{{\bm{x}}'} \subset \tilde{D}$ because
 $\xi$ is an arbitrary positive number.
Thus, it holds that  $\text{cl} (\tilde{D}) \subset \tilde{D} $.
 Therefore, we have $\tilde{D} = \text{cl} (\tilde{D} )$.
Finally, by using the fact that the closure is a closed set, $\tilde{D}$ is also closed.
\end{proof}

%% file: subsec3-3-utf8.tex
\subsection{Proof of Theorem \ref{thm:owaru}}\label{sub3}
\begin{proof}
Let $\xi $ be a positive number.
Then, from {\sf (A2)}, with probability 1, there exists $\delta _\xi >0$ such that  \eqref{eq:as1mean} holds.
Next, let
\begin{align*}
    \hat{\gamma}^2_t ({\bm{x}} ) &= \int _D \hat{\Phi}_{\bm{s}} \left ( 1-  \hat{\Phi}_{\bm{s}}  \right ) g( {\bm{s}} | {\bm\theta}_{\bm{x}} ) d {\bm{s}}, \\
    \hat{\Phi}_{\bm{s}} &=\Phi \left ( \frac{h-\hat{\mu}_t ({\bm{s}}) }{\hat{\sigma}_t ({\bm{s}})}    \right ),
\end{align*}
and we use notation $\hat{\mu}^{-1}_{t, \delta_\xi} = \hat{\mu}^{-1}_t (    (h -\delta _\xi , h+\delta_\xi ) )$.
Then, for each ${\bm{x}} \in \mathcal{X}$, $\hat{\gamma}^2_t ({\bm{x}} ) $ satisfies
\begin{align}
\hat{\gamma}^2_t ({\bm{x}} )  &=    \int _D \hat{\Phi}_{\bm{s}} \left ( 1-  \hat{\Phi}_{\bm{s}} \right ) g( {\bm{s}} | {\bm\theta}_{\bm{x}} ) d {\bm{s}}  \nonumber \\
&= \int _{D \setminus \hat{\mu}^{-1}_{t, \delta_\xi}} \hat{\Phi}_{\bm{s}} \left ( 1-  \hat{\Phi}_{\bm{s}} \right ) g( {\bm{s}} | {\bm\theta}_{\bm{x}} ) d {\bm{s}}    \nonumber \\
& \q + \int _{D \cap \hat{\mu}^{-1}_{t, \delta_\xi} } \hat{\Phi}_{\bm{s}} \left ( 1-  \hat{\Phi}_{\bm{s}} \right ) g( {\bm{s}} | {\bm\theta}_{\bm{x}} ) d {\bm{s}}  \nonumber \\
&\leq
\int _{D \setminus \hat{\mu}^{-1}_{t, \delta_\xi} } \hat{\Phi}_{\bm{s}} \left ( 1-  \hat{\Phi}_{\bm{s}} \right ) g( {\bm{s}} | {\bm\theta}_{\bm{x}} ) d {\bm{s}}    \nonumber \\
& \q + \int _{ \hat{\mu}^{-1}_{t, \delta_\xi} } g( {\bm{s}} | {\bm\theta}_{\bm{x}} ) d {\bm{s}}  \nonumber \\
&\leq
\int _{D \setminus \hat{\mu}^{-1}_{t, \delta_\xi} } \hat{\Phi}_{\bm{s}} \left ( 1-  \hat{\Phi}_{\bm{s}} \right ) g( {\bm{s}} | {\bm\theta}_{\bm{x}} ) d {\bm{s}}    \nonumber \\
& \q + \max_{ {\bm{x}}' \in \mathcal{X} } \int _{ \hat{\mu}^{-1}_{t, \delta_\xi} } g( {\bm{s}} | {\bm\theta}_{{\bm{x}}'} ) d {\bm{s}} . \label{eq:int_step1}
\end{align}
Note that $\gamma^2_t ({\bm{x}})$ is equal to an observed value of the random variable $\hat{\gamma}^2_t ({\bm{x}})$.
For any element ${\bm{s}}$ satisfying
$${\bm{s}}   \in D \setminus \hat{\mu}^{-1}_t (    (h -\delta _\xi , h+\delta_\xi ) ) ,$$
it holds that  $h-\hat{\mu}_t ( {\bm{s}} ) \geq \delta_\xi $ or
$h-\hat{\mu}_t ( {\bm{s}} ) \leq - \delta_\xi $.
Moreover, noting that $1-\Phi(a) = \Phi(-a)$, we get
\begin{align*}
& \quad \Phi \left ( \frac{h-\hat{\mu}_t ({\bm{s}}) }{\hat{\sigma}_t ({\bm{s}})}    \right ) \left \{ 1-  \Phi \left ( \frac{h-\hat{\mu}_t ({\bm{s}}) }{\hat{\sigma}_t ({\bm{s}})}    \right )   \right \} \\
& \leq
\Phi \left ( \frac{h-\hat{\mu}_t ({\bm{s}}) }{\hat{\sigma}_t ({\bm{s}})}    \right ) \left \{ 1-  \Phi \left ( \frac{\delta_\xi }{\hat{\sigma}_t ({\bm{s}})}    \right )   \right \} \\
&\leq
\Phi \left ( \frac{h-\hat{\mu}_t ({\bm{s}}) }{\hat{\sigma}_t ({\bm{s}})}    \right )  \Phi \left ( \frac{-\delta_\xi }{\hat{\sigma}_t ({\bm{s}})}    \right )   \\
& \leq \Phi \left ( \frac{-\delta_\xi }{\hat{\sigma}_t ({\bm{s}})}    \right )
\end{align*}
 when  $h-\hat{\mu}_t ( {\bm{s}} ) \geq \delta_\xi $.
Similarly, if   $h-\hat{\mu}_t ( {\bm{s}} ) \leq - \delta_\xi $, we obtain
\begin{align*}
& \quad \Phi \left ( \frac{h-\hat{\mu}_t ({\bm{s}}) }{\hat{\sigma}_t ({\bm{s}})}    \right ) \left \{ 1-  \Phi \left ( \frac{h-\hat{\mu}_t ({\bm{s}}) }{\hat{\sigma}_t ({\bm{s}})}    \right )   \right \} \\
& \leq
\Phi \left ( \frac{-\delta_\xi }{\hat{\sigma}_t ({\bm{s}})}    \right ) \left \{ 1-  \Phi \left ( \frac{\delta_\xi }{\hat{\sigma}_t ({\bm{s}})}    \right )   \right \} \\
&\leq
 \Phi \left ( \frac{-\delta_\xi }{\hat{\sigma}_t ({\bm{s}})}    \right ).
\end{align*}
Therefore, \eqref{eq:int_step1} can be expressed as
\begin{align}
\hat{\gamma}^2_t ({\bm{x}} ) & \leq
\int _{D \setminus \hat{\mu}^{-1}_{t, \delta_\xi}  }  \Phi \left ( \frac{-\delta_\xi }{\hat{\sigma}_t ({\bm{s}})}    \right ) g( {\bm{s}} | {\bm\theta}_{\bm{x}} ) d {\bm{s}} \nonumber \\
 &\hspace{70pt}  + \max_{ {\bm{x}}' \in \mathcal{X} } \int _{ \hat{\mu}^{-1}_{t, \delta_\xi}   } g( {\bm{s}} | {\bm\theta}_{{\bm{x}}'} ) d {\bm{s}}  \nonumber \\
 & \leq
\int _{D  }  \Phi \left ( \frac{-\delta_\xi }{\hat{\sigma}_t ({\bm{s}})}    \right ) g( {\bm{s}} | {\bm\theta}_{\bm{x}} ) d {\bm{s}} \nonumber \\
 &\hspace{70pt} + \max_{ {\bm{x}}' \in \mathcal{X} } \int _{ \hat{\mu}^{-1}_{t, \delta_\xi} } g( {\bm{s}} | {\bm\theta}_{{\bm{x}}'} ) d {\bm{s}}  \nonumber \\
& =
\int _{\tilde{D}  }  \Phi \left ( \frac{-\delta_\xi }{\hat{\sigma}_t ({\bm{s}})}    \right ) g( {\bm{s}} | {\bm\theta}_{\bm{x}} ) d {\bm{s}}  \nonumber \\
 &\hspace{70pt} + \max_{ {\bm{x}}' \in \mathcal{X} } \int _{ \hat{\mu}^{-1}_{t, \delta_\xi} } g( {\bm{s}} | {\bm\theta}_{{\bm{x}}'} ) d {\bm{s}}  \nonumber \\
& \leq
\int _{\tilde{D}  }  \max_{ {\bm{s}}' \in \tilde{D} }\Phi \left ( \frac{-\delta_\xi }{\hat{\sigma}_t ({\bm{s}}')}    \right ) g( {\bm{s}} | {\bm\theta}_{\bm{x}} ) d {\bm{s}} \nonumber \\
&\hspace{70pt} + \max_{ {\bm{x}}' \in \mathcal{X} } \int _{ \hat{\mu}^{-1}_{t, \delta_\xi}  } g( {\bm{s}} | {\bm\theta}_{{\bm{x}}'} ) d {\bm{s}}  \nonumber \\
& \leq
\max_{ {\bm{s}}' \in \tilde{D} }\Phi \left ( \frac{-\delta_\xi }{\hat{\sigma}_t ({\bm{s}}')}    \right ) \int _{\tilde{D}  }   g( {\bm{s}} | {\bm\theta}_{\bm{x}} ) d {\bm{s}} \nonumber \\
 &\hspace{70pt} + \max_{ {\bm{x}}' \in \mathcal{X} } \int _{ \hat{\mu}^{-1}_{t, \delta_\xi}   } g( {\bm{s}} | {\bm\theta}_{{\bm{x}}'} ) d {\bm{s}}  \nonumber \\
& \leq
\max_{ {\bm{s}}' \in \tilde{D} }\Phi \left ( \frac{-\delta_\xi }{\hat{\sigma}_t ({\bm{s}}')}    \right ) \nonumber \\
 &\hspace{50pt} + \max_{ {\bm{x}}' \in \mathcal{X} } \int _{ \hat{\mu}^{-1}_{t, \delta_\xi}  } g( {\bm{s}} | {\bm\theta}_{{\bm{x}}'} ) d {\bm{s}}.  \label{eq:int_step2}
\end{align}
Furthermore, since the right hand side in \eqref{eq:int_step2} does not depend on ${\bm{x}} \in \mathcal{X} $, we have
\begin{align}
\max _{ {\bm{x}} \in \mathcal{X} } \hat{\gamma}^2_t ({\bm{x}} )
& \leq
\max_{ {\bm{s}}' \in \tilde{D} }\Phi \left ( \frac{-\delta_\xi }{\hat{\sigma}_t ({\bm{s}}')}    \right ) \nonumber \\
&  + \max_{ {\bm{x}}' \in \mathcal{X} } \int _{ \hat{\mu}^{-1}_t (    (h -\delta _\xi , h+\delta_\xi ) )   } g( {\bm{s}} | {\bm\theta}_{{\bm{x}}'} ) d {\bm{s}} . \label{eq:int_step3}
\end{align}

Next, let $a$ be a positive number with  $\Phi (-\delta _\xi /a^{1/2} )  < \xi $.
Then, from  {\sf (A4)}, for any ${\bm{x}} \in \tilde{D}$, there exists  $\delta _{a/2,{\bm{x}} } >0$ such that
 $| \sigma^2_t ({\bm{x}} )  - \sigma^2_t ({\bm{x}}') | < a/2$  for any
${\bm{x}}' \in \mathscr{N} ({\bm{x}};\delta _{a/2,{\bm{x}} })$.
Furthermore, we define the following family of open sets:
$$
\{   \mathscr{N} ({\bm{x}};\delta _{a/2,{\bm{x}} }) \ | \ {\bm{x}} \in \tilde{D} \} \equiv \mathscr{U}.
$$
Note that $\mathscr{U}$ is an open cover of $\tilde{D} $.
In addition, from Lemma \ref{lem:compact},
$\tilde{D}$ is compact.
Hence, $\mathscr{U}$ has a finite subcover
$$
\mathscr{U}' \equiv \{      \mathscr{N} ({\bm{x}}_i;\delta _{a/2,{\bm{x}}_i }) \ | \ i=1,\ldots, U, \ {\bm{x}}_i \in \tilde{D} \}    \subset \mathscr{U}.
$$
Based on $\mathscr{U}' $, we define $\Omega' = \{   x_1, \ldots , x_U \}$.
 Then, $\Omega'$ is a finite subset of $\tilde{D}$ and satisfies
\begin{align}
\tilde{D}   \subset  \bigcup _{ {\bm{x}} \in \Omega' }      \mathscr{N} ({\bm{x}};\delta _{a/2,{\bm{x}} }). \label{eq:finite_sub_covering}
\end{align}
On the other hand, from  Lemma \ref{lem:as1cov}, with probability 1, for any $ {\bm{x}} \in \Omega' $ it holds that
$$
\sigma^2_t ({\bm{x}} ) \to 0.
$$
Thus, for some sufficiently large $T$,    it holds that  $\sigma^2_T ({\bm{x}} ) <a/2 $  for any ${\bm{x}} \in \Omega'$.
In addition, noting that  $\mathscr{N} ({\bm{x}};\delta _{a/2,{\bm{x}} })$ satisfies
$$
| \sigma^2_t ({\bm{x}} ) - \sigma^2 _t ({\bm{x}}') | <a/2,
$$
we get
\begin{align}
\sigma^2_T ({\bm{x}}') < a/2 +  \sigma^2_T ({\bm{x}} )   < a/2+a/2 =a. \label{eq:upper_a}
\end{align}
Hence,  for any
${\bm{x}} \in \Omega'$ and ${\bm{x}}' \in \mathscr{N} ({\bm{x}};\delta _{a/2,{\bm{x}} })$,
it holds that   $\sigma^2_T ({\bm{x}}' ) <a$.
Thus, using this inequality and \eqref{eq:finite_sub_covering}, we can show that $\sigma^2_T ({\bm{s}}' ) <a$  for any ${\bm{s}}' \in \tilde{D}$.
 Recall that the positive number $a$ satisfies $\Phi(-\delta_\xi /a^{1/2}) < \xi $.
Consequently, we obtain
$$
\max_{ {\bm{s}}' \in \tilde{D} }\Phi \left ( \frac{-\delta_\xi }{{\sigma}_T ({\bm{s}}')}    \right ) < \xi.
$$
This implies that
\begin{align}
\limsup_{t \to \infty} \max_{ {\bm{s}}' \in \tilde{D} }\Phi \left ( \frac{-\delta_\xi }{\hat{\sigma}_t ({\bm{s}}')}    \right )
 < \xi , \  ( \text{a.s.} ) . \label{eq:last}
\end{align}
Therefore, from \eqref{eq:as1mean}, \eqref{eq:last} and \eqref{eq:int_step3}, we have
\begin{align*}
\limsup_{t \to \infty} \max _{ {\bm{x}} \in \mathcal{X} } \hat{\gamma}^2_t ({\bm{x}} )
 <
2 \xi ,  \  ( \text{a.s.} ).
\end{align*}
In other words, with probability 1, there exists a number $N$ such that  $\max_{ {\bm{x}} \in \mathcal{X} }  \gamma^2_N ({\bm{x}} ) < 2 \xi$.

Finally, from the definition of the classification rule, each point ${\bm{x}} \in \mathcal{X}$ is classified to
 $\mathcal{H}_t$ or $\mathcal{L}_t$ if
 $\beta^{1/2} \gamma_t ({\bm{x}} ) < \epsilon$.
Hence, if
$\max_{ {\bm{x}} \in \mathcal{X} }  \gamma^2_t ({\bm{x}} ) < \epsilon^2 \beta^{-1} $,
all points are classified.
Therefore, since $\xi$ is any positive number, putting
$\xi = 2^{-1} \epsilon^2 \beta^{-1} $ we have Theorem \ref{thm:owaru}.
\end{proof}

%% file: paper.bbl
\begin{thebibliography}{}

\bibitem[Beland and Nair, 2017]{beland2017bayesian}
Beland, J.~J. and Nair, P.~B. (2017).
\newblock Bayesian optimization under uncertainty.
\newblock In {\em NIPS BayesOpt 2017 workshop}.

\bibitem[Bishop, 2006]{bishop2006pattern}
Bishop, C.~M. (2006).
\newblock {\em Pattern recognition and machine learning}.
\newblock springer.

\bibitem[Bogunovic et~al., 2016]{bogunovic2016truncated}
Bogunovic, I., Scarlett, J., Krause, A., and Cevher, V. (2016).
\newblock Truncated variance reduction: A unified approach to bayesian
  optimization and level-set estimation.
\newblock In {\em Advances in neural information processing systems}, pages
  1507--1515.

\bibitem[Bryan et~al., 2006]{bryan2006active}
Bryan, B., Nichol, R.~C., Genovese, C.~R., Schneider, J., Miller, C.~J., and
  Wasserman, L. (2006).
\newblock Active learning for identifying function threshold boundaries.
\newblock In {\em Advances in neural information processing systems}, pages
  163--170.

\bibitem[Dua and Graff, 2017]{Dua:2019}
Dua, D. and Graff, C. (2017).
\newblock {UCI} machine learning repository.

\bibitem[Gessner et~al., 2019]{gessner2019active}
Gessner, A., Gonzalez, J., and Mahsereci, M. (2019).
\newblock Active multi-information source bayesian quadrature.
\newblock In {\em Proceedings of the Thirty-Fifth Conference on Uncertainty in
  Artificial Intelligence, {UAI} 2019, Tel Aviv, Israel, July 22-25, 2019},
  page 245.

\bibitem[Girard et~al., 2003]{girard2003gaussian}
Girard, A., Rasmussen, C.~E., Candela, J.~Q., and Murray-Smith, R. (2003).
\newblock Gaussian process priors with uncertain inputs application to
  multiple-step ahead time series forecasting.
\newblock In {\em Advances in neural information processing systems}, pages
  545--552.

\bibitem[Gotovos et~al., 2013]{gotovos2013active}
Gotovos, A., Casati, N., Hitz, G., and Krause, A. (2013).
\newblock Active learning for level set estimation.
\newblock In {\em Twenty-Third International Joint Conference on Artificial
  Intelligence}.

\bibitem[Inatsu et~al., 2019]{inatsu2019active}
Inatsu, Y., Karasuyama, M., Inoue, K., and Takeuchi, I. (2019).
\newblock Active learning for level set estimation under cost-dependent input
  uncertainty.
\newblock {\em arXiv preprint arXiv:1909.06064}.

\bibitem[Kaya et~al., 2012]{inproceedings}
Kaya, H., Tufekci, P., and Gurgen, Fikret, S. (2012).
\newblock Local and global learning methods for predicting power of a combined
  gas \& steam turbine.
\newblock In {\em Proceedings of the International Conference on Emerging
  Trends in Computer and Electronics Engineering ICETCEE 2012}, pages 13--18.

\bibitem[O'Hagan, 1991]{o1991bayes}
O'Hagan, A. (1991).
\newblock Bayes--hermite quadrature.
\newblock {\em Journal of statistical planning and inference}, 29(3):245--260.

\bibitem[Oliveira et~al., 2019]{oliveira2019bayesian}
Oliveira, R., Ott, L., and Ramos, F. (2019).
\newblock Bayesian optimisation under uncertain inputs.
\newblock In Chaudhuri, K. and Sugiyama, M., editors, {\em Proceedings of
  Machine Learning Research}, volume~89 of {\em Proceedings of Machine Learning
  Research}, pages 1177--1184. PMLR.

\bibitem[Papoulis and Pillai, 2002]{papoulis2002probability}
Papoulis, A. and Pillai, S.~U. (2002).
\newblock {\em Probability, random variables, and stochastic processes}.
\newblock Tata McGraw-Hill Education.

\bibitem[Rasmussen and Williams, 2006]{gpml}
Rasmussen, C.~E. and Williams, C. K.~I. (2006).
\newblock {\em Gaussian Processes for Machine Learning}.
\newblock MIT Press.

\bibitem[Settles, 2009]{settles2009active}
Settles, B. (2009).
\newblock Active learning literature survey.
\newblock Technical report, University of Wisconsin-Madison Department of
  Computer Sciences.

\bibitem[Shahriari et~al., 2016]{shahriari2016taking}
Shahriari, B., Swersky, K., Wang, Z., Adams, R.~P., and De~Freitas, N. (2016).
\newblock Taking the human out of the loop: A review of bayesian optimization.
\newblock {\em Proceedings of the IEEE}, 104(1):148--175.

\bibitem[Srinivas et~al., 2010]{srinivas2009gaussian}
Srinivas, N., Krause, A., Kakade, S., and Seeger, M. (2010).
\newblock Gaussian process optimization in the bandit setting: No regret and
  experimental design.
\newblock In F{\"u}rnkranz, J. and Joachims, T., editors, {\em Proceedings of
  the 27th International Conference on Machine Learning (ICML-10)}, pages
  1015--1022, Haifa, Israel. Omnipress.

\bibitem[Sui et~al., 2015]{sui2015safe}
Sui, Y., Gotovos, A., Burdick, J., and Krause, A. (2015).
\newblock Safe exploration for optimization with gaussian processes.
\newblock In {\em International Conference on Machine Learning}, pages
  997--1005.

\bibitem[Sui et~al., 2018]{DBLP:conf/icml/SuiZBY18}
Sui, Y., Zhuang, V., Burdick, J.~W., and Yue, Y. (2018).
\newblock Stagewise safe bayesian optimization with gaussian processes.
\newblock In {\em {ICML}}, volume~80 of {\em Proceedings of Machine Learning
  Research}, pages 4788--4796. {PMLR}.

\bibitem[Tufekci, 2014]{article}
Tufekci, P. (2014).
\newblock Prediction of full load electrical power output of a base load
  operated combined cycle power plant using machine learning methods.
\newblock {\em International Journal of Electrical Power \& Energy Systems},
  60:126^^e2^^80^^93140.

\bibitem[Xi et~al., 2018]{xi2018bayesian}
Xi, X., Briol, F.-X., and Girolami, M. (2018).
\newblock {B}ayesian quadrature for multiple related integrals.
\newblock In Dy, J. and Krause, A., editors, {\em Proceedings of the 35th
  International Conference on Machine Learning}, volume~80 of {\em Proceedings
  of Machine Learning Research}, pages 5373--5382, Stockholmsm^^c3^^a4ssan,
  Stockholm Sweden. PMLR.

\bibitem[Zanette et~al., 2018]{zanette2018robust}
Zanette, A., Zhang, J., and Kochenderfer, M.~J. (2018).
\newblock Robust super-level set estimation using gaussian processes.
\newblock In {\em Joint European Conference on Machine Learning and Knowledge
  Discovery in Databases}, pages 276--291. Springer.

\end{thebibliography}
